\documentclass[]{article}
\usepackage[toc,page,header]{appendix}
\RequirePackage{minitoc}
\usepackage{etoc}
\usepackage[utf8]{inputenc} 
\usepackage[T1]{fontenc}    
\usepackage{hyperref}       
\usepackage{url}            
\usepackage{booktabs}       
\usepackage{amsfonts}       
\usepackage{nicefrac}       
\usepackage{microtype}      
\usepackage{xcolor}         
\usepackage{amssymb}
\usepackage{amsmath}
\usepackage{amsthm}
\usepackage{bbm}
\newcommand\norm[1]{\left\lVert#1\right\rVert}

\usepackage{xfrac}
\usepackage{nicefrac}
\usepackage{mathtools}
\usepackage{tikz}
\usetikzlibrary{trees}
\usepackage{geometry}
\usepackage{graphicx}
\usepackage[export]{adjustbox}
\usepackage{authblk}
\usepackage{hyperref}

\newtheorem{proposition}{Proposition}[section]

\newtheorem{remark}{Remark}[section]

 \usepackage[round]{natbib}
 \renewcommand{\bibname}{References}
 
\usepackage{booktabs}
\usepackage{url}

\newtheorem{theorem}{Theorem}[section]
\newtheorem{corollary}{Corollary}[section]
\newtheorem{lemma}{Lemma}[section]
\begin{document}

\newcommand{\best}[1]{\textcolor{blue}{\textbf{\underline{#1}}}}
\newcommand{\second}[1]{\textcolor{red}{\textbf{#1}}}
\newcommand{\third}[1]{\textcolor{black}{#1}}
\newcommand{\github}{\url{https://github.com/KrishnaswamyLab/blis}}

\title{BLIS-Net: Classifying and Analyzing Signals on Graphs}
\author{
  Charles Xu\textsuperscript{2,4},
  Laney Goldman\textsuperscript{3}, 
  Valentina Guo\textsuperscript{2}, 
  Benjamin Hollander-Bodie\textsuperscript{1}, 
  Maedee Trank-Greene\textsuperscript{5}, 
  Ian Adelstein\textsuperscript{1}, 
  Edward De Brouwer\textsuperscript{4}, 
  Rex Ying\textsuperscript{2}, 
  Smita Krishnaswamy\textsuperscript{2,4,*,\textdagger}, and 
  Michael Perlmutter\textsuperscript{6,7,*}
}

\date{}

\maketitle

{\centering
\textsuperscript{1}Department of Mathematics, Yale University, \\
\textsuperscript{2}Department of Computer Science, Yale University,\\
\textsuperscript{3}Department of Mathematics, Harvey Mudd College, \\
\textsuperscript{4}Department of Genetics, Yale School of Medicine, \\
\textsuperscript{5}Department of Applied Mathematics, University of Colorado Boulder, \\
\textsuperscript{6}Department of Mathematics, Boise State University, \\
\textsuperscript{7}Program in Computing, Boise State University \\
\textsuperscript{*}Co-senior author \\
\textsuperscript{\textdagger}Corresponding author: \href{mailto:smita.krishnaswamy@yale.edu}{smita.krishnaswamy@yale.edu} \\
\par}

\begin{abstract}

Graph neural networks (GNNs) have emerged as a powerful tool for tasks such as node classification and graph classification. However, much less work has been done on signal classification, where the data consists of many functions (referred to as signals) defined on the vertices of a single graph. These tasks require networks designed differently from those designed for traditional GNN tasks. Indeed, traditional GNNs  rely on localized low-pass filters, and signals of interest may have intricate multi-frequency behavior and exhibit long range interactions. 
This motivates us to introduce the BLIS-Net (Bi-Lipschitz Scattering Net), a novel GNN that builds on the previously introduced geometric scattering transform. Our network is able to capture both local and global signal structure and is able to capture both low-frequency and high-frequency information. We make several crucial changes to the original geometric scattering architecture which we prove increase the ability of our network to capture information about the input signal and show that BLIS-Net achieves superior performance on both synthetic and real-world data sets based on traffic flow and fMRI data.

\end{abstract}

\doparttoc 
\faketableofcontents 

\parttoc 

\section{Introduction}
Recent years have seen a tremendous rise of Graph Neural Networks (GNNs) as a powerful tool for processing graph-structured data \citep{wu2020comprehensive}. Most of the research has focused on three families of tasks: graph-level tasks in which the data consists of many different graphs \citep{errica2019fair}, node-level tasks \citep{kipf2016semi} such as classifying each user in a large network, and edge-level tasks such as link prediction \citep{zhang2018link}. However, there is another family of problems, signal-level tasks,  which has received comparatively little attention. 

Here we focus on developing a high-performing network for signal classification, where the goal is to predict the label $y_i$ of a signal (function) $\mathbf{x_i}:V\rightarrow \mathbb{R}$ defined on the vertices of a weighted graph $G=(V,E,w)$. Notably, this is the natural generalization of image classification, which can be thought of as classifying many signals defined on a grid graph, to more irregular domains. 
Additionally, we note that signal-level tasks have many practical applications. For example, in traffic data, the road network is kept fixed but the number of cars at each intersection varies each day. A natural goal would be to identify, and then analyze, anomalous traffic patterns. In the analysis of brain-scan data, a patient's neuronal structures (i.e., voxels, neurons) can be modeled as a fixed graph with different levels of activity in each region across time, offering a useful framework for analyzing neural data \citep{li2021braingnn}.

However, GNNs that have been designed with node-level tasks in mind perform limited processing from a signal perspective. A foundational principle of most node-level analysis is homophily \citep{zhu2020beyond}, the idea that nodes are similar to their neighbors. Therefore, one wants to produce a hidden representation of each node which varies slowly among neighbors. However, when focused on signal-level tasks, the local homophily heuristic is not applicable and we find that it is important to capture (i) both low-frequency and high-frequency information in the input signal and (ii) both local and global behavior.

A natural choice when working with signals is to use methods from graph signal processing \citep{shuman:emerging2013} incorporated into a neural network. To this end, we rely on the geometric scattering transform \citep{gao:graphScat2018,gama:diffScatGraphs2018,gama:stabilityGraphScat2019,zou:graphCNNScat2018}, a multi-order, multi-scale transform that alternates wavelet transforms and non-linear modulus activations in the form of a deep (although typically fixed) network. The wavelets act as band-pass filters that capture information at different frequencies and scales. Therefore, geometric scattering provides  a solid starting place for signal-level tasks. 

However, geometric scattering alone is insufficient for two key reasons:   First, we establish that geometric scattering is not injective due to its modulus operation, and thus loses expressivity and ability to distinguish between certain signal classes. Second, scattering produces an unnecessarily high-dimensional representation of the signal that is not tuned to classification purposes. This motivates us to introduce BLIS-Net (Bi-Lipschitz Scattering Net) which incorporates advances to address these issues while facilitating integration into larger neural networks.  
BLIS-Net builds on previous work on the geometric scattering transform and introduces ReLU and reflected ReLU activations to preserve injectivity. Further, BLIS-Net incorporates dimension reduction and classification modules to demonstrate the modular use of bi-Lipschitz Scattering within a larger ML context. Our contributions can be summarized as follows:

\begin{enumerate}
    \item We introduce BLIS-Net, a novel GNN for signal classification on graphs.
    \item We prove two theorems (Theorem \ref{thm: not injective new} and Theorem \ref{prop: bi-Lipschitz}), which, when considered jointly, show that the BLIS module is provably more expressive than the geometric scattering transform. In particular, Theorem \ref{prop: bi-Lipschitz} shows that BLIS is bi-Lipschitz and therefore stably invertible. 
    \item We show that BLIS-Net achieves superior performance on both synthetic data and  real-world data sets derived from traffic and fMRI data.
   
\end{enumerate}

\section{Background}

\subsection{Notation and Preliminaries}\label{sec: notation}

Let $G=(V,E,w)$ be a  weighted, connected graph with vertices $V=\{ v_1, \cdots, v_n\}$. 
Throughout this paper, we will consider functions $\mathbf{x}:V\rightarrow\mathbb{R}$, which we refer to as \emph{graph signals}, and will in a slight abuse of notation not distinguish between the signal $\mathbf{x}$ and the vector $\mathbf{x}\in\mathbb{R}^n$ defined by $x_i=\mathbf{x}(v_i)$. We will let $A$ be the weighted adjacency of $G$, let $\mathbf{d}=A\mathbbm{1}$ denote the weighted degree vector, and let $D=\text{diag}(\mathbf{d})$ be the  degree matrix. 

We will let 
$L_{N} = I - D^{-1/2} A D^{-1/2}$
denote the symmetric normalized graph Laplacian. It is well-known that $L_N$ is positive semidefinite and admits an orthonormal basis (ONB) of eigenvectors with  $L_N\mathbf{v_i}=\omega_i\mathbf{v_i}$ with $0=\omega_1<\omega_2\leq \ldots\leq\omega_n\leq2$ (where the fact that $\omega_2>0$ follows from the assumption that $G$ is connected). This allows us to write   
$L_N = V\Omega V^T$, where $V$ is a matrix whose $i$-th column is $\mathbf{v_i}$ and $\Omega$ is a diagonal matrix with $\Omega_{i,i}=\omega_i$. Since the $\{\mathbf{v_i}\}_{i=1}^n$ form an ONB, we see that $V$ is unitary and $V^TV=I$. 

It is known (e.g., Section 2 of \citet{min2022can}) that \begin{equation}\label{eqn: quadratic form} \mathbf{x}^TL_N\mathbf{x}=\sum_{\{v_i,v_j\}\in E} (\tilde{x}_i-\tilde{x}_j)^2\end{equation}
where $\tilde{\mathbf{x}}\coloneqq\mathbf{D}^{-1/2}\mathbf{x}$ is a normalized version of $\mathbf{x}$. Therefore $L_N$ is viewed as a matrix whose quadratic form measures the smoothness of (normalized) signals. If we take $\mathbf{x}=\mathbf{v_i}$ we have $\mathbf{v_i}^TL_N\mathbf{v_i}=\omega_i$. Therefore, we may interpret each eigenvalue $\omega_i$ as a frequency and each eigenvector as a generalized Fourier mode. The high-frequency eigenvectors oscillate rapidly within local neighborhoods leading to large values in \eqref{eqn: quadratic form} whereas the low-frequency eigenvectors are smooth in the sense they vary slowly within graph neighborhoods.  Therefore, we view methods based on the eigendecomposition of the graph Laplacian as the natural extension of traditional signal processing to the graph setting. 
 
We note that since $V$ is unitary, we have $p(L_N)=Vp(\Omega)V^T$ for any polynomial $p$. Thus, for suitably nice functions $f:[0,\infty)\rightarrow\mathbb{R}$ and diagonalizable matrix $M=B\Xi B^{-1}$ with $\Xi=\text{diag}(\xi_1,\ldots,\xi_n)$, we define \begin{equation}
 \label{eqn: spectral calculus}f(M) = B f(\Xi) B^{-1},
 \end{equation}
 where $f(\Xi) = \mathrm{diag}( 
 f(\xi_1),\ldots,
 f(\xi_n)).
$

\subsection{Graph signals and signal-level tasks}\label{sec: background on signals}

Graph signal-level tasks naturally arise in biological, natural, and social systems. 
Some key examples include:

\begin{itemize}
\item Predicting properties of social networks may also be formulated as a signal-level task. 
For instance, while classifying the political affiliation of an individual is a node-level task,  characterizing a polarized population (low-frequency) vs a non-polarized population (high-frequency) is a signal-level task.
\item Networks that occur in nature such as cell-communication networks have genes or cytokines as signals on a fixed graph substrate~\citep{moon:PHATE2017}. In many of these cases, the number of signals on the network is close to the number of nodes. 
\item In neuroscience, one can represent brain measurements as signals living on a fixed graph. The graph embodies the connectivity between different brain regions and the signal would be the brain activity measurements in each brain region. A typical task is then to predict external stimuli from brain signals~\citep{menoret2017evaluating}.
\end{itemize}
In general, a signal-level task is any machine learning task, e.g., classification, regression, or clustering, where the data set consists of many different signals defined on a single fixed graph.

\subsection{Diffusion Matrices }\label{sec: diffusion matrices}

Let $g(t)$ be a decreasing function on $[0,2]$ with $g(0) = 1, g(2)=0$, and let $T = g\left( L_N\right)$ (defined as in \eqref{eqn: spectral calculus}). By construction, $T$ is diagonalizable and $T=V\Lambda V^T$, where $\Lambda \coloneqq g\left( \Omega \right)$. As our primary example, which we will use in all of our numerical experiments, we will let
\begin{equation}\label{eqn: g canonical}
g(t) = 1 - \frac{t}{2}.
\end{equation} $T$ then becomes the symmetrized diffusion operator
\begin{align*}
    T & = I - \frac{L_N}{2} 
     = \frac{1}{2}\left(I + D^{-1/2} A D^{-1/2} \right).
\end{align*}
Next, we  let $\mathbf{w}\in\mathbb{R}^n$ be a weight vector with $w_i>0$, let $W\coloneqq\text{diag}(\mathbf{w})$, and  $K$ the asymmetric diffusion matrix 
\begin{equation}\label{eqn: K}
K \coloneqq W^{-1} T W.
\end{equation}
We let $\mathbf{L}^2_{\mathbf{w}}$ be the weighted inner product space with 
$$\langle \mathbf{x},\mathbf{y} \rangle_{\mathbf{w}}= \langle W\mathbf{x},W\mathbf{y}\rangle_{2}=\sum_{i=1}^nx_iy_iw_i,$$
and norm denoted by $\norm{\cdot}_{\mathbf{w}}$. One may verify that $K$ is self-adjoint on $\mathbf{L}^2_{\mathbf{w}}$ (see \cite{perlmutter2023understanding}, Lemma 1.1).

We note that if we set $W = D^{-1/2}$, then  $K$ becomes the  lazy random walk matrix, $
    P  \coloneqq \frac{1}{2}\left(I + A D^{-1} \right),$
which was used to construct diffusion wavelets in {\citet{gao:graphScat2018},
whereas if we set $W=I,$ $K$ is simply equal to the matrix $T$ which was used in 
\citet{gama:diffScatGraphs2018}.
More generally, \citet{perlmutter2023understanding} considered $W=D^{\alpha}$, $-.5\leq \alpha\leq.5$ and found emprically that the optimal choice of $\alpha$ varied from task to task.

\subsection{Graph Wavelets and Frame Properties}\label{sec: wavelets}

Let $J\geq 0$, and let $\mathcal{F}=\{F_{j}\}^{J+1}_{j=0}$ be a collection of $n\times n$ matrices. We say that $\mathcal{F}$ is a \emph{frame} on $\mathbf{L}^2_\mathbf{w}$ if there exist constants $0< c\leq C<\infty $ such that,
\begin{equation}\label{eqn: Frame condition}
    c\|\mathbf{x}\|_\mathbf{w}^{2} \leq \sum_{j=0}^{J+1}\norm{F_{j}\mathbf{x}}^{2}_\mathbf{w} \leq C\|\mathbf{x}\|_\mathbf{w}^2,\quad\forall \mathbf{x}\in \mathbb{R}^n.
\end{equation}
We say that $\mathcal{F}$ is a \emph{non-expansive frame} if $C\leq 1$ and that $\mathcal{F}$ is an \emph{isometry} if $c=C=1$.

We now construct two families of wavelet frames. Analogous to \citet{tong2022learnable}, let $\{s_j\}_{j=0}^{J+1}$ be a sequence of scales with $s_{0} = 0$ and $s_{1} = 1$, and $s_j < s_{j+1}$. For $0\leq j\leq J$, let $p_{j} (t)$ denote the polynomial defined by
$$p_{j} (t) \coloneqq 
    t^{s_{j}} - t^{s_{j+1}}$$ and let $p_{J+1}\coloneqq t^{s_{J+1}}$.
By construction, each $p_j$ is  nonnegative on $[0,1]$, and therefore, we may define $q_j(t) \coloneqq p_j(t)^{1/2}$ for $0\leq t\leq 1.$
We then define two graph wavelet transforms $$\mathcal{W}^{(1)}_{J}  \coloneqq \{ \Psi_{j}^{(1)},\Phi_J^{(1)} \}_{0 \leq j\textcolor{black}{\leq J}},\quad\Psi_j^{(1)}\coloneqq q_j(K),\quad\Phi_J^{(1)}\coloneqq q_{J+1}(K)$$ (where $q_j(K)$ is defined as in  \eqref{eqn: spectral calculus})
and $$\mathcal{W}^{(2)}_{J} \coloneqq \{ \Psi_{j}^{(2)}, \Phi_J^{(2)} \}_{0 \leq j \leq J},\quad \Psi_j^{(2)}\coloneqq p_j(K), \quad \Phi_J^{(2)}\coloneqq p_{J+1}(K).$$ 
The following result shows that $\mathcal{W}^{(1)}_J$ is an isometry and that $\mathcal{W}^{(2)}_{J}$ is a non-expansive frame on 
$\mathbf{L}^2_{\mathbf{w}}$. For a proof, please see Appendix \ref{prf: frame}. 
\begin{proposition} \label{prop: frame}For any $\mathbf{x}\in\mathbb{R}^n$, we have,
\begin{equation}
    \label{eqn: W1 isometry}
\|\mathcal{W}^{(1)}_{J}\mathbf{x}\|^2_{\mathbf{w}}\coloneqq \| \Phi_J^{(1)}\mathbf{x} \|_{\mathbf{w}}^2 + \sum_{j=0}^J \|\Psi_j^{(1)} \mathbf{x} \|_{\mathbf{w}}^2 =\| \mathbf{x} \|_{\mathbf{w}}^2.
\end{equation}  
Additionally, there exists a constant $c>0$, depending only on the maximal scale $s_{J+1}$, such that 
    \begin{equation}\label{eqn: W2frame}
    c \| \mathbf{x} \|_{\mathbf{w}}^2 \leq \|\mathcal{W}^{(2)}_{J}\mathbf{x}\|^2_{\mathbf{w}}\leq  \| \mathbf{x} \|_{\mathbf{w}}^2\quad \text{for all }\mathbf{x}\in\mathbb{R}^n,
\end{equation}
where $\|\mathcal{W}^{(2)}_{J}\mathbf{x}\|^2_{\mathbf{w}}$ is defined analogously to $\|\mathcal{W}^{(1)}_{J}\mathbf{x}\|^2_{\mathbf{w}}$. 
   \end{proposition}

\subsection{The Graph Scattering Transform}\label{sec: scatttering}

Given a wavelet frame $\mathcal{W}_J=\{\Psi_j\}_{j=0}^J\cup\{\Phi_J\}$  such as  $\mathcal{W}^{(1)}_J$ and $\mathcal{W}^{(2)}_J$, the graph scattering transform is a multilayer feedforward network consisting of alternating wavelet transforms and non-linearities (building off of an analogous construction \citep{mallat:scattering2012} modeling CNNs for Euclidean data such as images). In particular, given a sequence of scales $j_1,\ldots,j_m$, we define
\begin{equation}\label{eqn: U}
U[j_1,\ldots,j_m]\mathbf{x}=M\Psi_{j_m}\ldots M\Psi_{j_1}\mathbf{x},
\end{equation}
where $M\mathbf{x}=|\mathbf{x}|$ is the componentwise modulus (absolute value) operator $(M\mathbf{x})_i=|x_i|$. Then, after computing each of the $U[j_1,\ldots,j_m]\mathbf{x}$, one may extract $m$-th order scattering coefficients via the low-pass filter $\Phi_J$,
$$
S_J[j_1,\ldots,j_m]\mathbf{x}=\Phi_JU[j_1,\ldots,j_m]\mathbf{x}.
$$ 
If one wishes to apply the graph scattering transform to tasks such as node classification, they may compute the scattering coefficients up to order $M$  and take the scattering coefficients evaluated at each vertex, $\{S_J[j_1,\ldots,j_m]\mathbf{x}(v): 0\leq m \leq M, 0\leq j_1,\ldots, j_m
\leq J\}$ 
as a collection of node features which may then be input into another machine learning algorithm  such as a  a multilayer perceptron. Alternatively, if one wishes to apply the graph scattering transform to whole-graph level tasks such as graph classification, one first performs a global aggregation (e.g., summation) before applying the final classifier. Importantly, we note that coefficients of orders $m=0,\ldots,M$ (where the zeroth-order coefficient is simply $\Phi_J\mathbf{x}$) are fed into the classifier, not just the $m$-th order scattering coefficients.

\section{BLIS-Net}
Here, we introduce BLIS-Net a novel neural network for graph signals, which as discussed in Section \ref{sec: background on signals}, arise frequently in the natural and behavioral sciences. 
In order to create a network that can classify or regress properties of signals, one needs to create a rich representation of the signal. A natural choice for such a representation is a signal processing transform such as the geometric scattering transform discussed in Section \ref{sec: scatttering}. Indeed, it was shown that the geometric scattering transform was an effective tool for identifying anomalies in traffic data in \citet{bodmann2022scattering}. However, the geometric scattering transform has limitations in its ability to characterize its input signal, which motivates us to introduce the BLIS module.

The primary deficiency of the geometric scattering transform which we seek to address is lack of injectivity. Since the scattering transform is constructed using the modulus in \eqref{eqn: U}, it is trivial that the scattering transform will produce identical representations of $\mathbf{x}$ and $-\mathbf{x}$. The following theorem shows that there are also non-trivial examples of distinct signals with identical scattering coefficients. This may be proved by constructing signals $\mathbf{x_1}$ and $\mathbf{x_2}$, where each $\mathbf{x_j}$ is supported on two disjoint regions, such that $\mathbf{x_1}\neq\pm\mathbf{x_2}$, but $\mathbf{x_1}$ and $\mathbf{x_2}$ have identical scattering coefficients.  We also note that in Section \ref{sec: experiments synthetic}, we conduct experiments on synthetic data modeled after this choice of $\mathbf{x_1}$ and $\mathbf{x_2}$ to further illustrate the limitations of the geometric scattering transform which are addressed by BLIS.

\begin{theorem}\label{thm: not injective new}
    Under certain assumptions, there exist signals $\mathbf{x_1}$ and $\mathbf{x_2}$ with identical scattering coefficients such that $\mathbf{x_1}\neq\pm\mathbf{x_2}$\footnote{See Appendix \ref{sec: proof of noninjectivity} for a proof and detailed theorem statement.}.
\end{theorem}

This result shows that the geometric scattering transform has limits on its expressive power since there are non-equivalent signals of which it produces identical representations. Notably, the importance of injectivity has also been noted in the context of graph classification. In particular, \citet{xu2018how} showed that using an injective aggregation function (in a message-passing network) was the key to producing a maximally expressive graph neural network.
We also note that the result of Theorem \ref{thm: not injective new} is somewhat surprising since \cite{mallat:waveletPhaseRetrieval2015} showed  there were no non-trivial signal pairs with identical scattering coefficients for the original Euclidean scattering transform \citep{mallat:scattering2012}.

\subsection{The BLIS Module}\label{sec: BLIS}

Recall from Section \ref{sec: scatttering} that the geometric scattering transform 
uses an alternating sequence of wavelet transforms and componentwise modulus operators  to produce coefficients such as 
$S_J[j_1,j_2]\mathbf{x}=\Phi_JM\Psi_{j_2}M\Psi_{j_1}\mathbf{x}$.  BLIS makes the following modifications: 

\begin{enumerate}
    \item To induce injectivity, BLIS uses two different activation functions $\sigma_1(\mathbf{x})\coloneqq\text{ReLU}(\mathbf{x})$ and $\sigma_2(\mathbf{x})\coloneqq\text{ReLU}(-\mathbf{x})$. Notably, we have 
\begin{equation*}
\sigma_1(\mathbf{x})+\sigma_2(\mathbf{x})=M\mathbf{x}.
\end{equation*}
Thus, the use of  $\sigma_1$ and $\sigma_2$ may be viewed as decomposing the  absolute value into two disjointly supported non-linearities. 
    Additionally, this modification 
    is crucial to proving Theorem \ref{prop: bi-Lipschitz} which shows that the BLIS module is injective on $\mathbb{R}^n$.
    \item To account for all frequency bands of the signal, BLIS uses the entire wavelet frame $\mathcal{W}_J=\{ \Psi_{j}\}_{j=0}^J\cup\{\Phi_J\}$ in each layer. This is in contrast to the geometric scattering transform which
    does not utilize the low-pass filter $\Phi_J$ until after the final non-linearity. This modification is needed to ensure that BLIS has the bi-Lipshitz property established in Theorem \ref{prop: bi-Lipschitz} and is also key to the conservation of energy property established in Theorem \ref{prop: U nonexpansive}. This latter property ensures that BLIS doesn't lose information which may be critical for tasks such as classification. 
    \item Since BLIS uses the entire wavelet frame in each layer, all of the energy of the input signal is preserved in each layer. Therefore, the only  output of an $m$-layer BLIS module is the coefficients produced in the final layer (i.e., through a sequence of $m$ filterings followed by activations). This is in contrast to the geometric scattering transform which outputs first-order coefficients $S_J[j_1]\mathbf{x},$ second-order coefficients $S_J[j_1,j_2]\mathbf{x}$, etc., up to $m$-th order coefficients $S_J[j_1,\ldots,j_m]\mathbf{x}$ (in addition to a single zeroth-order coefficient which is simply $\Phi_J\mathbf{x}$). This makes it straightforward to incorporate the BLIS module into a neural network without the need for skip connections. \label{difference no skip}
\end{enumerate}

\begin{figure*}
        \centering
    \includegraphics[width=\linewidth]{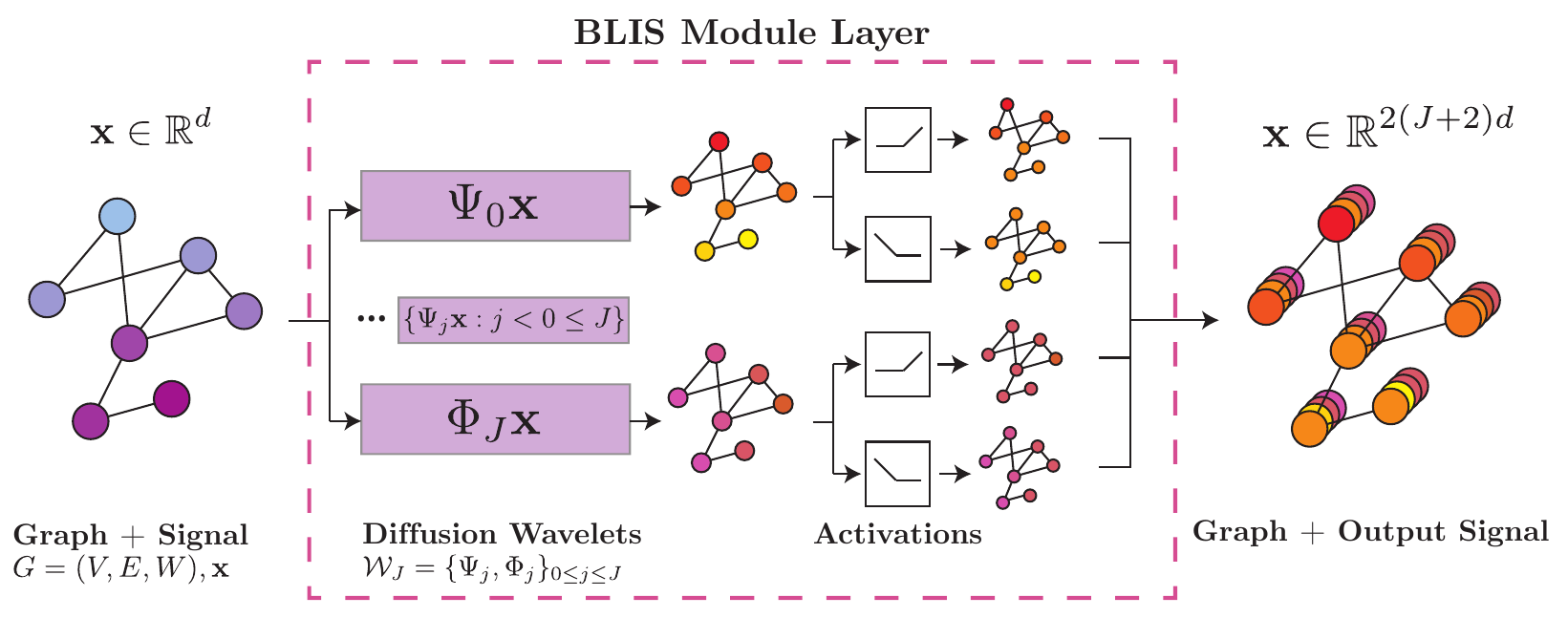}
    \caption{The BLIS module: We first apply multiscale diffusion wavelet transform to the input signal and then two activation functions, $\sigma_1$ and $\sigma_2$. The output  is a multivariate signal, with $2(J+2)$ times the original dimension.}
    \label{fig:architecture}
\end{figure*}

To explicitly define the BLIS module, we rewrite the wavelet frame $\mathcal{W}_J=\{\Psi_j\}_{j=0}^J\cup\{\Phi_J\}$ as  $\mathcal{F}=\{F_j\}_{j=0}^{J+1}$ where $F_j=\Psi_j$ for $0\leq j\leq J$ and $F_{J+1}=\Phi_J$. We let $m\geq 1$ denote the depth of the network and define  
\begin{equation}\label{eqn: all the mth order 
coeffs}
B[j_{1},k_{1},\cdots,j_{m},k_{m}](\mathbf{x})
    \coloneqq\sigma_{k_{m}}(F_{{j}_{m}}\sigma_{k_{m-1}}(F_{{j}_{m-1}}\cdots \sigma_{k_{1}}(F_{{j}_{1}}\mathbf{x}))\cdots)
\end{equation}
for  $0\leq j_i\leq J+1,$ and $k_i\in\{1,2\}.$
We then let $\mathbf{B}_m(\mathbf{x})$ denote the set of all the $B[j_1,k_1,\ldots,j_m,k_m]$.

We remark that one could readily modify the BLIS framework to include other frames $\mathcal{F}$ in place of the wavelets $\mathcal{W}^{(1)}_J$ or $\mathcal{W}^{(2)}_J$. For example, one could use the spectral wavelets considered in \citet{zou:graphCNNScat2018} or frames obtained as the union of different wavelet families. Importantly, the proofs of Theorems \ref{prop: bi-Lipschitz} and \ref{prop: U nonexpansive} do not depend on the specific wavelet construction, but only on the frame constants $0<c\leq C<\infty$. Therefore, both of these results apply to variations of BLIS constructed via arbitrary $\mathcal{F}$ satisfying \eqref{eqn: Frame condition}.

\subsection{The bi-Lipschitz Property}

Theorem \ref{thm: not injective new}, stated above, shows that the geometric scattering transform is not injective on $\mathbb{R}^n$ (even up to the equivalence relation $\mathbf{x}\sim\pm\mathbf{x}$). Therefore, it may lack the ability to effectively characterize graph signals. By contrast, the following theorem shows that BLIS is a bi-Lipschitz map on weighted inner product space $\mathbf{L}^2_\mathbf{w}$ introduced in Section \ref{sec: diffusion matrices} where we equip the image space with the mixed norm obtained by taking the (unweighted) $\ell^2$ norm of the weighted $\ell^2$ norms of the individual  $B[j_1,k_1,\ldots,j_m,k_m]\mathbf{x}$, so that 
\begin{equation*}
     \norm{\mathbf{B}_m (\mathbf{x})}_{\mathbf{w},2}^2 
     = \sum_{k_{i}=1}^{2} \sum_{j_{i}=0}^{J+1}\norm{B[j_{1},k_{1},\cdots,j_{m},k_{m}](\mathbf{x})}_{\mathbf{w}}^2.
\end{equation*}

\begin{theorem} \label{prop: bi-Lipschitz}
$\mathbf{B}_m$ is bi-Lipshitz on $\mathbf{L}^2_\mathbf{w}$, i.e.,  
    $$\left(\frac{c}{2}\right)^m\norm{\mathbf{x}-\mathbf{y}}_{\mathbf{w}}^2 \leq \norm{\mathbf{B}_{m}(\mathbf{x})-\mathbf{B}_{m}(\mathbf{y})}_{\mathbf{w},2}^2 \leq C^{m}\norm{\mathbf{x}-\mathbf{y}}_{\mathbf{w}}^2$$
    for all $\mathbf{x},\mathbf{y}\in\mathbb{R}^n$, where $0<c\leq C<\infty$ are the  frame bounds for the wavelets defined as in \eqref{eqn: Frame condition}.
\end{theorem}
For a proof of Theorem \ref{prop: bi-Lipschitz}, please see Appendix \ref{prf: bi-Lipschitz}.
The following corollary is immediate from the first inequality in Theorem \ref{prop: bi-Lipschitz} and the fact that $\|\cdot\|_{\mathbf{w},2}$ is
 a norm.
\begin{corollary}
$\mathbf{B}_m$ is injective on $\mathbb{R}^n$.
\end{corollary}

We note that the lower bound established in Theorem \ref{prop: bi-Lipschitz} implies the existence of a Lipschitz continuous inverse map\footnote{Discussed further in Appendix \ref{app: inverse}.} that reconstructs $\mathbf{x}$ from $\mathbf{B}_m(\mathbf{x})$. This property is particularly interesting in light of work \citep{zou:graphScatGAN2019,Castro2020,BhaskarGCPK22} which has attempted to invert the geometric scattering transform as part of an encoder-decoder graph-generation network.

\subsection{ Properties inherited from scattering}

In addition to the bi-Lipschitz property, we may also show that BLIS retains desirable theoretical properties from  geometric scattering such as permutation equivariance and conservation of energy.

Permutation equivariance  is the property that if we reorder the vertices $v_1,\ldots, v_n$ (and therefore reorder order the entries of the input signal since $x_i=\mathbf{x}(v_i)$), then the representations of the vertices are reordered in the same manner. It is crucial to the success of a well-designed  GNN since it ensures that the network captures the intrinsic graph structure of the data rather than relying on the ordering of the vertices. The following theorem shows BLIS is  permutation equivariant. For a proof please see Appendix \ref{sec: proof equivariance}. 

\begin{theorem}\label{thm: equivariance}
    Let $\Pi$ be a permutation matrix corresponding to a reordering of the nodes. Then, 
    \begin{equation*}
\Pi B[j_{1},k_{1},\cdots,j_{m},k_{m}]\mathbf{x}=B[j_{1},k_{1},\cdots,j_{m},k_{m}]\Pi\mathbf{x},   \end{equation*}
for all $j_1,k_1,\ldots,j_m,k_m$,
where on the right-hand side $B[j_1,k_1,\ldots,j_m,k_m]$ is defined in terms of the permuted ordering (with the permuted weight vector $\Pi\mathbf{w}$).
\end{theorem}

In our aggregation module (discussed below in Section \ref{sec: BLISnet}), we will perform a global summation over the vertices. In light of Theorem \ref{thm: equivariance}, we will be summing the same terms on both the original and the permuted graph, just in a different order. Therefore, the output of the aggregation module will be the same for both graphs. Thus, the BLIS module produces an equivariant representation of the signal from which the  aggregation module extracts invariance.

Previous work has shown that the infinite-depth geometric scattering transform preserves the norm of the input. For example, Theorem 3.5 of \citet{perlmutter2023understanding} shows that if the scattering transform is constructed using $\mathcal{W}_J^{(1)}$,  then  
\begin{equation*}
\sum_{m=0}^\infty\sum_{0\leq j_i\leq J} \|S_J[j_1,\ldots,j_m]\mathbf{x}\|_\mathbf{w}^2=\|\mathbf{x}\|_\mathbf{w}^2.
\end{equation*}

We may derive an analogous result for BLIS. However, our theorem differs from previous work in that it shows that the energy of the input signal is preserved in each layer (whereas previous work showed energy was conserved when summing over all layers). Indeed, this result helps motivate the third modification discussed at the beginning of Section \ref{sec: BLIS}, where we implement BLIS without skip connections, unlike  scattering.
\begin{theorem}\label{prop: U nonexpansive}
    For all $\mathbf{x}\in\mathbb{R}^n$, we have  $$c^{m}\norm{\mathbf{x}}_{\mathbf{w}}^2 \leq \norm{\mathbf{B}_{m}(\mathbf{x})}_{\mathbf{w},2}^2 \leq C^{m}\norm{\mathbf{x}}_{\mathbf{w}}^2,$$
    where $0<c\leq C<\infty$ are the  frame bounds for the wavelets defined as in \eqref{eqn: Frame condition}.
    In particular, if $c=C=1,$ as is the case for $\mathcal{W}^{(1)}_J$, we have $\norm{\mathbf{B}_{m}(\mathbf{x})}_{\mathbf{w},2}=\|\mathbf{x}\|_{\mathbf{w}}$.
\end{theorem}
\noindent For a proof of Theorem \ref{prop: U nonexpansive}, please see Appendix \ref{prf: U nonexpansive}.

\subsection{BLIS-Net Architecture}\label{sec: BLISnet}
The  bi-Lipschitz Scattering Network (BLIS-Net) integrates the BLIS module with several other modules for ML purposes:

\begin{enumerate}
    \item \textbf{BLIS module layer}: The first layer of our network utilizes the BLIS module to extract features from the input graph signal.
    \item \textbf{Moment aggregation module}: 
For  each $j_1,k_1,\ldots,j_m,k_m$, we aggregate of the BLIS features across the nodes, i.e., $$B'[j_{1},k_{1},\cdots,j_{m},k_{m}]\mathbf{x_i}=\sum_{v\in V} B'[j_{1},k_{1},\cdots,j_{m},k_{m}](\mathbf{x_i})(v).$$ 
    \item \textbf{Embedding layer}:
    To reduce the risk of overfitting and increase computational efficiency, we next perform a dimensionality reduction via an embedding layer. 
    \item \textbf{Classification layer}: Finally we include an MLP classifier featuring softmax activation. 
\end{enumerate}

We note that although the focus of this paper is signal classification, other common machine learning tasks such as clustering and regression have natural analogs in the signal setting. Our method can be flexibly adapted to these tasks due to its modular design. For example, one could train a clustering algorithm on top of the output of the BLIS module. \begin{table}[]
\centering
\begin{tabular}{@{}l|cc@{}}
\toprule
Synthetic     & Different $\mu$ & Same $\mu$     \\ \midrule
GCN           & $99.0 \pm 0.4$  & $91.7 \pm 2.0$ \\
GAT           & $99.2 \pm 0.5$  & $91.6 \pm 2.0$ \\
GIN           & \second{$\mathbf{99.5 \pm 0.2}$}  & $91.3 \pm 1.4$ \\
GPS           & $95.4 \pm 5.9$  & \second{$\mathbf{97.7 \pm 0.9}$} \\
\midrule
Scattering (W1) & $97.7 \pm 1.0$ & $96.5 \pm 1.2$ \\ 
Scattering (W2) & $88.3 \pm 4.3$ & $96.8 \pm 1.0$ \\ 
\midrule
BLIS-Net (W1)       & \best{$\mathbf{100.0 \pm 0.0}$} & \second{$\mathbf{97.7 \pm 0.5}$} \\
BLIS-Net (W2)       & \second{$\mathbf{99.5 \pm 0.3}$}  & \best{$\mathbf{98.6 \pm 0.4}$} \\
\bottomrule
\end{tabular}
\label{tab: synthetic results main}
\caption{Accuracy on the synthetic data sets. }
\end{table}

\section{Experimental  Results
}
We demonstrate the utility of BLIS-Net (with both  $\mathcal{W}^{(1)}_J$ and $\mathcal{W}^{(2)}_J$) on synthetic and real-world data sets. 
As baselines, we use several widely adopted graph neural networks based on the message-passing framework: the Graph Convolutional Network (GCN) \citep{kipf2016semi}, Graph Attention Network (GAT)  \citep{velivckovic2017graph}, and the Graph Isomorphism Network (GIN) \citep{xu2018how}. We  also consider the general, powerful, scalable (GPS) graph transformer \citep{rampavsek2022recipe} which has achieved state-of-the-art performance on a wide range of benchmarks. We note that, unlike message-passing GNNs, the GPS allows information to spread across the graph via full connectivity, thus allowing the network to capture global properties of the signal. In our tables, we color the \best{top-performing} and \second{second-best} method. Further details on model implementation, computational complexity, data sets, hyperparameters, and training procedures are provided in the appendix.

\subsection{Synthetic data}\label{sec: experiments synthetic}
We first generate $2N$ random functions $f^{(1)}_1,\ldots, f^{(1)}_N$ and $f^{(2)}_1,\ldots, f^{(2)}_N$ defined on $[0,1]^2$ from two different  distributions. We then define graph signals  $\mathbf{x_k^{(j)}}$ with $(x_k^{(j)})_i=f^{(j)}_k(v_i)$ where the vertices $v_1,\ldots, v_n$ are chosen-uniformly at random from $[0,1]^2$ and connected to their $k$-nearest neighbors. 
Our goal is to predict which  distribution the signal was generated from.

 In particular, we consider two families of functions $$f_j^{(1)}=g_{\mu_1,\sigma_1}+g_{\mu_2,\sigma_2},\quad  f^{(2)}_j=g_{\mu_1,\sigma_1}-g_{\mu_2,\sigma_2},$$ where $g_{\mu,\sigma}(x)\coloneqq \exp\left(-\frac{\|x-\mu^{j}_1\|^2_2}{2(\sigma_1^j)^2}\right)$ is a Gaussian function with center $\mu$ and bandwidth $\sigma$. In our first set of experiments, we generate values of $\mu^j_1$ and $\mu^j_2$ uniformly at random from $[0,1]^2$, set $\sigma^j_1=\sigma^j_2=\sigma^j$, where $\sigma_j$ is chosen uniformly from $[0,1]$. Our second setup is similar, but with $\mu_1^j=\mu_2^j,$ $\sigma_2=\sigma_1/2$.

The first setting, $\mu_1\neq \mu_2$,  results in signals modeled after those used in the proof of Theorem \ref{thm: not injective new} with support concentrated near two, possibly far away, points. Therefore, we unsurprisingly observe that BLIS, as well as GAT, GIN, GCN, and GPS, perform well on this task whereas  scattering is the least accurate method, likely due to its use of the absolute value.\footnote{ Code needed to reproduce our experiments is 
available at~\github. Experiments were performed on a computing cluster with 8 CPUs and 4 NVIDIA RTX 5000 GPUs.}

In our second setting, $\mu_1=\mu_2=\mu$ and $\sigma_2=\sigma_1/2$,
 we  view the Gaussians as two signals interfering with each other. 
Unlike the first setup, the absolute value does not severely limit the ability of the scattering transform to distinguish the signal classes. Indeed, the primary difference between these two signal classes is an oscillatory pattern in the middle of the signals' support. We observe that the two wavelet-based methods (Scattering and BLIS) are well equipped to capture this signal oscillation and both outperform GIN, GAT, and GCN  (with BLIS outperforming scattering due to its increased expressive power). The utility of wavelets for this task is visualized in Figure \ref{fig:filter response}, where we show that the two signal classes have markedly different responses to wavelet filters, but compartively similar responses to the low-pass filter used in GCN.
Further details on our experimental setup as well as ablation studies, where we also consider the geometric scattering transform and the BLIS module paired with shallow classifiers, e.g., SVM, are presented in Appendix \ref{app: ablation}. 

\begin{figure}[htbp]
    \centering
    \includegraphics[width=0.7\linewidth]{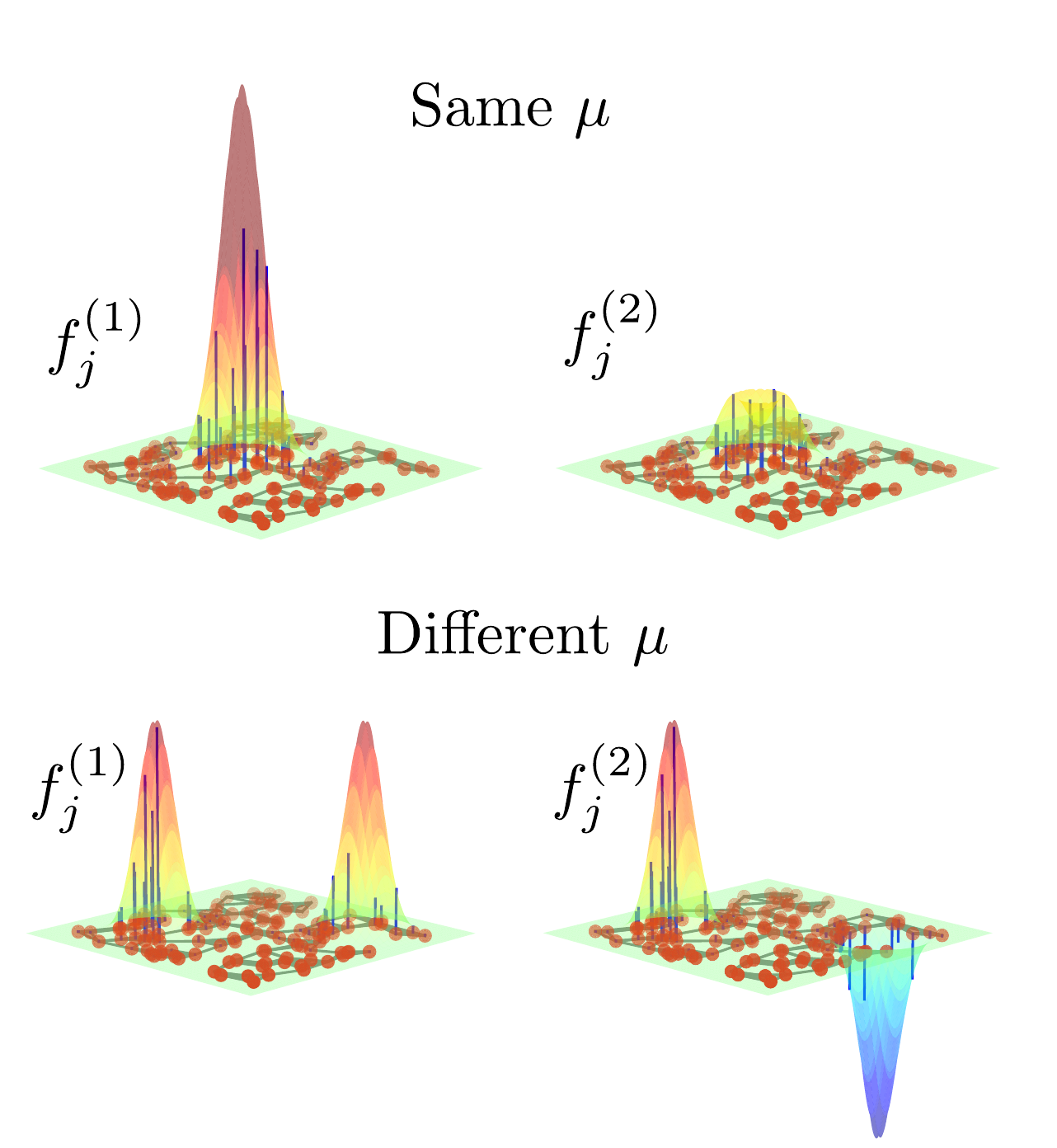}
    \caption{Synthetic signals $f^{(1)}_j$ and $f^{(2)}_j$.}
    \label{fig:enter-label}
\end{figure}

\begin{figure}[htbp]
    \centering
    \includegraphics[width=.7\linewidth]{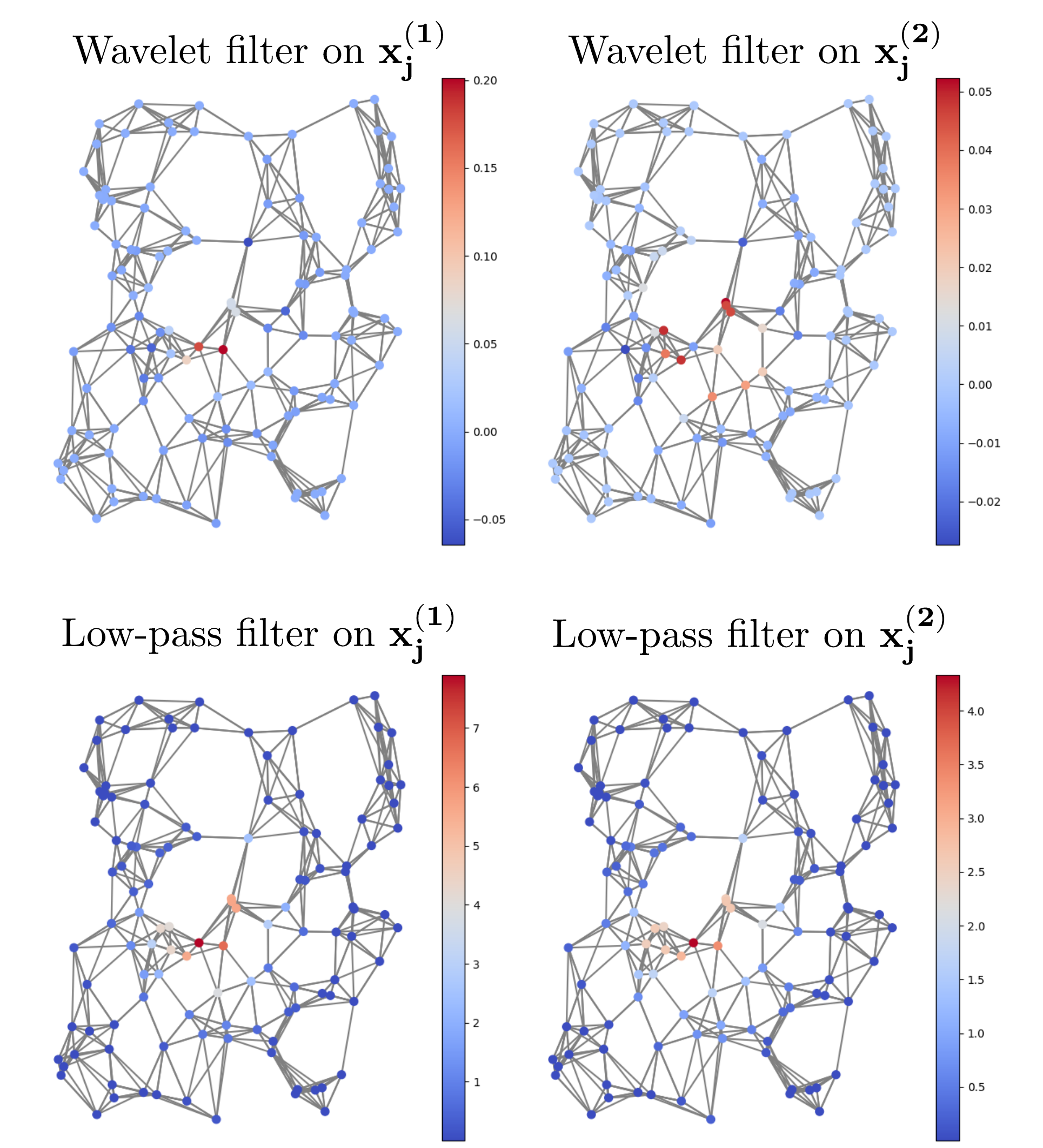}
    \caption{Filter responses on $\mathbf{x_j^{(1)}}$ and $\mathbf{x_j^{(2)}}$.}
    \label{fig:filter response}
\end{figure}
\subsection{Caltrans Traffic data}\label{sec: traffic data set}

\begin{table}[]
\centering
\begin{adjustbox}{width=0.7\linewidth}
\begin{tabular}{l|ccc}
\toprule
PEMS03        & HOUR            & DAY            & WEEK           \\ \hline
GCN           & $27.8 \pm 2.0$  & $14.1 \pm 0.1$ & $30.8 \pm 0.5$ \\
GAT           & $27.4 \pm 1.6$  & $14.1 \pm 0.3$ & $30.8 \pm 0.5$ \\
GIN           & $14.0 \pm 12.4$ & $14.3 \pm 0.4$ & $30.8 \pm 0.5$ \\ 
GPS           & $57.4 \pm 0.1$ & $49.6 \pm 0.3$ & $31.9 \pm 0.4$ \\
\midrule
Scattering (W1) & $58.2 \pm 0.8$ & $45.6 \pm 0.7$ & $46.4 \pm 1.3$ \\ 
Scattering (W2) & $60.4 \pm 0.5$ & $49.5 \pm 0.9$ & $51.4 \pm 1.0$ \\
\midrule
BLIS-Net (W1)      & \second{$\mathbf{63.1 \pm 2.2}$}  & \second{$\mathbf{53.1 \pm 1.3}$} & \second{$\mathbf{54.8 \pm 1.8}$} \\
BLIS-Net (W2)       & \best{$\mathbf{68.3 \pm 2.1}$}  & \best{$\mathbf{56.3 \pm 0.0}$} & \best{$\mathbf{61.7 \pm 2.6}$} \\
\bottomrule
\end{tabular}
\end{adjustbox}
\caption{Accuracy on the PEMS03 traffic data set. 
}
\label{tab:pems03}
\end{table}

\begin{table}[h!]
\centering
\begin{adjustbox}{width=0.7\linewidth}
\begin{tabular}{l|ccc}
\toprule
PEMS07        & HOUR            & DAY            & WEEK           \\ \hline
GCN           & $27.4 \pm 2.0$  & $14.6 \pm 0.6$ & $28.5 \pm 0.5$ \\
GAT           & $26.8 \pm 4.1$  & $14.6 \pm 0.6$ & $28.6 \pm 0.5$ \\
GIN           & $14.3 \pm 12.6$ & $15.8 \pm 0.8$ & $28.4 \pm 0.6$ \\
GPS           &  $39.9 \pm 2.7$ &  $27.7 \pm 1.9$ & $30.4 \pm 0.6$ \\
\midrule
Scattering (W1) & $54.0 \pm 0.6$ & $53.3 \pm 0.9$ & $56.9 \pm 1.3$ \\ 
Scattering (W2) & $54.3 \pm 0.7$ & $55.2 \pm 1.2$ & $61.6 \pm 1.0$ \\
\midrule
BLIS-Net (W1)       &  \best{$\mathbf{63.5 \pm 1.1}$}  & \best{$\mathbf{72.9 \pm 1.5}$} & \second{$\mathbf{76.8 \pm 2.0}$} \\
BLIS-Net (W2)       & \second{$\mathbf{63.4 \pm 2.1}$}  & \second{$\mathbf{71.0 \pm 2.4}$} & \best{$\mathbf{77.3 \pm 1.6}$} \\
\bottomrule
\end{tabular}
\end{adjustbox}
\caption{Accuracy on the PEMS07 traffic data set. 
}
\label{tab:pems07}
\end{table}

Here we consider  data consisting  of highway traffic measurements collected by the Caltrans Performance Measurement System (PeMS) \citep{chen2001freeway}, where over 39,000 sensors are deployed across California highways and data are aggregated every five minutes.
 PeMS03 and PeMS07  consist of traffic data from California's 3rd and 7th congressional districts and provide two months of consecutive traffic data collected between 2016 and 2018, depending on the data set\footnote{Additional experiments on data from the 4th and 8th district are provided in the appendix along with additional details on our experimental setup.}. We aim to predict the hour of the day (24 classes), the day of the week (7 classes), and the week of the month (4 classes) that each traffic observation corresponds to. 
On both PEMS03 and PEMS07, we observe that the message-passing-based methods (GCN, GAT, and GIN) perform poorly. Both wavelet-based methods (Scattering and BLIS-Net) perform well, with BLIS-Net outperforming Scattering perhaps because of  its improved expressive power. The GPS graph transformer performs better than the message-passing based methods but less well than the wavelet-based methods, particularly when attempting to predict the week.

\subsection{Partly Cloudy fMRI data set}

Modeling functional magnetic resonance imaging (fMRI) data as a graph is a useful computational approach for analyzing brain signals; the nodes are brain regions of interest (ROIs) whose connections can be defined in multiple ways. The fMRI data provides graph signals in the form of a blood oxygenation level dependent (BOLD) signal across the ROIs. 

Here, we utilize a data set collected from participants who were shown Disney Pixar's "Partly Cloudy" in \citet{richardson2018development}. 
 39 ROIs were extracted from the fMRI data, and the graph connectivity was created using a $k$-nearest neighbors graph based on the centroids of the ROIs.  
We consider the problem of using the fMRI data to classify the emotional state of the animated film, delineating the frames into three classes, positive, negative, and neutral emotions. Similarly to \citet{busch2023multi}, we applying temporal smoothing at each node \footnote{Results without smoothing are in Appendix \ref{app: ablation}.}. BLIS-Net outperforms all other methods. As with the PeMS data sets, scattering is the second best performing method followed by the GPS graph transformer, underscoring the value of capturing global information as well as the full frequency 
spectrum of the input signal.


\begin{table}[]
\centering
\begin{tabular}{@{}l|c@{}}
\toprule
Partly Cloudy & Emotion classification \\ \midrule
GCN           & $39.3 \pm 5.9$         \\
GAT           & $39.3 \pm 6.0$        \\
GIN           & $42.1 \pm 6.0$        \\ 
GPS           & $56.4 \pm 4.3$ \\
\midrule
Scattering (W1) & $60.6 \pm 4.9$ \\ 
Scattering (W2) & $62.3 \pm 5.1$ \\ 
\midrule
BLIS-Net (W1)       & \second{$\mathbf{67.1 \pm 4.3}$}        \\
BLIS-Net (W2)       & \best{$\mathbf{68.3 \pm 3.6}$}         \\
\bottomrule
\end{tabular}
\caption{ Accuracy on Partly Cloudy fMRI data.  }
\end{table}

\section{Conclusion and Future Work}

We have introduced BLIS-Net, a network for processing graph signals. The key piece of our architecture is the BLIS module, which modifies the geometric scattering transform in several ways in order to provably increase its expressive power for signal classification. We then show that BLIS-Net achieves superior to performance to both the original geometric scattering transform and other GNNs on both real and synthetic data.

There are also several natural avenues of future work. As alluded to in the intro, \citet{bodmann2022scattering} used a statistical analysis of the geometric scattering coefficients to detect anomalous traffic patterns. It is likely that similar tools can be used in conjunction with BLIS to detect anomalous signals on graphs such as traffic networks and brain-scan networks. Additionally, we note that many of our data sets have an implicit temporal structure in addition the graph structure. Therefore, developing a space-time version of BLIS, perhaps using techniques inspired by \cite{PanCO21}, would be a natural future direction. Lastly, we note that \cite{wenkel2022overcoming}  constructs a hybrid network which combines aspects of with more standard GCN-style networks and utilizes a localized attention mechanism to balance the two. We view hybridizations similar to this, with BLIS in place of the geometric scattering transform, as a potential avenue for improved numerical performance in future work.

 \bibliography{main}
 \newpage

\clearpage


\appendix
\onecolumn

\addcontentsline{toc}{section}{Appendix} 
\part{Appendix} 
\parttoc 

\section{Proof of Proposition \ref{prop: frame}}\label{prf: frame}

   The proof of Proposition \ref{prop: frame} follows by adapting the techniques used to prove Theorem 1 of \citet{tong2022learnable} and  Proposition 2.2
of \citet{perlmutter2023understanding} to more general scales (for $\mathcal{W}^{(1)}$) and more general diffusion matrices (for $\mathcal{W}^{(2)}$). 

\begin{proof}
\newcommand{\bW}{W}
\newcommand{\bV}{V}
\newcommand{\bx}{\mathbf{x}}
We will first prove \eqref{eqn: W1 isometry}, which estabilished that $\mathcal{W}^{(1)}_{J}$ is an isometry. 
Note that by \eqref{eqn: spectral calculus} and \eqref{eqn: K} we have 
\begin{align*}
    \Psi_{j}^{(1)} 
     =  q_j(K) 
     = \bW^{-1}\bV q_j(\Lambda)\bV^T\bW
\end{align*}
for all $0\leq j\leq J$ and also that 
\begin{align*}
    \Phi_J^{(1)} 
     =  q_{J+1}(K) = \bW^{-1}\bV q_{J+1}(\Lambda)\bV^T \bW.
\end{align*}
Additionally, we note that since $V$ is unitary, the definition of $\langle\cdot,\cdot\rangle_\mathbf{w}$ implies that we have
\begin{align*}
    \| \Psi_{j}^{(1)}\bx \|_{\mathbf{w}}^2 & = \langle \Psi_{j}^{(1)}\bx , \Psi_{j}^{(1)}\bx \rangle_\mathbf{w} 
     = \bx^T\bW^T\bV q_j(\Lambda)^2\bV^T\bW\bx,
\end{align*}
 and similarly $ \| \Phi_{j}^{(1)}\bx \|_{\mathbf{w}}^2  
     = \bx^T\bW^T\bV q_{J+1}(\Lambda)^2\bV^T\bW\bx$.
Therefore
\begin{align*}
    \|\mathcal{W}^{(1)}_{J}\mathbf{x}\|^2_{\mathbf{w}} &= \sum_{j=0}^J \| \Psi_{j}^{(1)}\bx \|_{\mathbf{w}}^2 + \| \Phi_J^{(1)}\bx \|_{\mathbf{w}}^2 \\
    & = \bx^T\bW^T\bV  \left[ \sum_{j=0}^{J+1}q_j(\Lambda)^2\right] \bV^T\bW\bx \\
    & = \bx^TW^T\bV Q_J(\Lambda)\bV^TW\bx \\
    & = \langle Q_J(\Lambda)\bV^T\bW\bx, \bV^T\bW\bx \rangle_2,
\end{align*}
where $Q_J(t) \coloneqq \sum_{j=0}^{J+1}q_j(t)^2$ (and $Q_J(\Lambda)$ is defined term by term along the diagonal according to \eqref{eqn: spectral calculus}).  Therefore the lower frame bound on $\mathcal{W}^{(1)}_{J}$ is given by
\begin{align*}
    c^{(1)}_J & \coloneqq \inf_{\mathbf{x}\neq0} \dfrac{\|\mathcal{W}^{(1)}_{J}\mathbf{x}\|^2_{\mathbf{w}}}{\|\bx\|^2_{\mathbf{w}}} \\
    & = \inf_{\mathbf{x}\neq0} \dfrac{\langle Q_J(\Lambda)\bV^T\bW\bx, \bV^T\bW\bx \rangle_2}{\|\bW\bx\|^2_{2}} \\
    & = \inf_{\mathbf{x}\neq0} \dfrac{\langle Q_J(\Lambda)\bV^T\bW\bx, \bV^T\bW\bx \rangle_2}{\|\bV^T\bW\bx\|^2_{2}} \label{b1}\tag{B.1}\\
    & = \inf_{\mathbf{y}\neq0} \dfrac{\langle Q_J(\Lambda) \mathbf{y}, \mathbf{y} \rangle}{\|\mathbf{y}\|^2_{2}} \label{b2}\tag{B.2}\\
    & = \min_{1\leq i \leq n} Q_J(\lambda_i) \label{b3}\tag{B.3}.\\
\end{align*}
\ref{b1} follows because  $\bV$ is unitary. \ref{b2} follows because $\bW$ is invertible and \ref{b3} follows because $Q_J(\Lambda)$ is diagonal with nonzero entries $Q_J(\lambda_i).$ We can similarly calculate the upper frame bound to be
\begin{align*}
    C^{(1)}_J & \coloneqq \sup_{\mathbf{x}\neq0} \dfrac{\|\mathcal{W}^{(1)}_{J}\mathbf{x}\|^2_{\mathbf{w}}}{\|\bx\|^2_{\mathbf{w}}} = \max_{1\leq i \leq n} Q_J(\lambda_i). 
\end{align*}
Using a telescoping sum, we see that
\begin{align}
    Q_J(t) & = \sum_{j=0}^{J+1}q_j(t)^2=
\sum_{j=0}^{J+1}p_j(t) 
    = t^{s_{J+1}}+\sum_{j=0}^{J}(t^{s_j}-t^{s_{j+1}})= t^{s_{J+1}} - (t^{s_{J+1}}-1) = 1 \label{eqn: telescope}
\end{align}
uniformly on $0\leq t\leq 1$. Thus, $c^{(1)}_J=C^{(1)}_J=1$ which concludes the proof of \eqref{eqn: W1 isometry}.

Turning our attention to \eqref{eqn: W2frame}, we may use the same logic as before to see that the lower and upper frame bounds for $\mathcal{W}^{(2)}_{J}$ are given by:
\begin{align*}
    c^{(2)}_J & \coloneqq \inf_{\mathbf{x}\neq0} \dfrac{\|\mathcal{W}^{(2)}_{J}\mathbf{x}\|^2_{\mathbf{w}}}{\|\bx\|^2_{\mathbf{w}}}
     = \min_{1\leq i \leq n} P_J(\lambda_i),
    \quad C^{(2)}_J  \coloneqq \sup_{\mathbf{x}\neq0} \dfrac{\|\mathcal{W}^{(2)}_{J}\mathbf{x}\|^2_{\mathbf{w}}}{\|\bx\|^2_{\mathbf{w}}}  = \max_{1\leq i \leq n} P_J(\lambda_i),
\end{align*}
where $P_J(t) \coloneqq \sum_{j=0}^{J+1}p_j(t)^2$. To determine the upper frame bound, it suffices to note that
\begin{align*}
    \max_{1\leq i \leq n} P_J(\lambda_i) & \leq \sup_{t\in [0,1]}\sum_{j=0}^{J+1}p_j(t)^2 \leq \sup_{t\in [0,1]} \left( \sum_{j=0}^{J+1}p_j(t)\right)^2 = 1
\end{align*}
where the second inequality comes from the positivity of each $p_j(t)$ and the final equality comes from the same reasoning as in \eqref{eqn: telescope}.
For the lower bound, we see that 
\begin{align*}
    \min_{1\leq i \leq n} P_J(\lambda_i) \geq \inf_{t\in [0,1]}\sum_{j=0}^{J+1}p_j(t)^2 \geq \inf_{t\in [0,1]}p_0(t)^2+p_{J+1}(t)^2=\inf_{t\in [0,1]}(1-t)^2+t^{2s_{J+1}}.
\end{align*}
The final quantity is a positive constant depending only on $s_{J+1}$. Thus, this completes the proof.

\end{proof}

\begin{remark}
In our experiments, we use dyadic scales, ($s_0=0, s_1=1, s_j=2^{j-1}, j\geq 2$). In this case,  Proposition 1 of \cite{Chew2022measure} implies that the lower frame bound $c$ for the $\mathcal{W}^{(2)}_J$ wavelets may be chosen to be a universal constant.
\end{remark}

\section{Details and proof for Theorem \ref{thm: not injective new}}\label{sec: proof of noninjectivity}

In this section, we provide full details on Theorem \ref{thm: not injective new} as well as some discussion. 

\subsection{Background - Wavelet Phase Retrieval}

The original scattering transform \citep{mallat:scattering2012} was introduced as a theoretical model for understanding the success of convolutional neural networks, defining scattering coefficients via an alternating sequence of wavelet convolutions and pointwise absolute values (moduluses):
$$
S_J[j_1,\ldots,j_m]f=\Phi_JM\Psi_{j_m}\ldots M\Psi_{j_1}f,\quad\text{for } f\in\mathbf{L}^2(\mathbb{R}^n).
$$
A natural question is to what extent do these coefficients determine a signal $f$? If two signals, $f_1$ and $f_2$ have the same coefficents, does this imply $f_1$ and $f_2$ coincide?

To answer this question, \citet{mallat:waveletPhaseRetrieval2015}, studied the descriptive power of the wavelet-modulus, $M\Psi_j$ which is the key building block of the scattering transform. Since $\Psi_j$ is linear, it is immediate that $M\Psi_jf_1=M\Psi_jf_2$ whenever $f_1=\pm f_2$ (or more generally when $f_1=e^{i\theta}f_2$ in the case the functions are complex-valued). The question then becomes whether this is the only setting in which $M\Psi_jf_1=M\Psi_jf_2$, i.e., are there any non-trivial ambiguities in the wavelet modulus. Questions such as this, whether or not a function can be determined (up to a global sign) by magnitude-only measurements, are known as the phase retrieval problems \citep{bandeira2014saving} and arise in wide variety of scientific domains including optics 
\citep{antonello2015modal}, astronomy \cite{fienup1987phase},
x-ray crystallography 
\citep{liu2012phase}, and speech-signal processing \citep{balan2006signal}.

The primary results  of \citet{mallat:waveletPhaseRetrieval2015} are (i) if the $\Psi_j$ are chosen to be Cauchy wavelets, then the wavelet modulus \emph{is} injective (up to the equivalence relation $f(x)\sim e^{i\theta}f(x)$) and therefore invertible. (ii) There is no uniform modulus of continuity on the inverse map, i.e., there are signals $f_1, f_2$ which are far apart (in the quotient metric induced by the relevant equivalence relation) such that $f_1$ and $f_2$ have nearly identical wavelet modulus.

Theorem \ref{thm: not injective} stated below, which is a more detailed version of Theorem \ref{thm: not injective new} from the main body, is meant to address the analogous question in the graph setting. Is the wavelet modulus invertible (up to a global sign change)? We show that under certain circumstances the answer to this question is no. There are signals $\mathbf{x_1}\neq\pm\mathbf{x_2}$ with identical wavelet moduluses therefore indentical scattering coefficients. 

Additionally, we also note that Theorem \ref{prop: bi-Lipschitz} is also partially motivated by the second result of \citet{mallat:waveletPhaseRetrieval2015}. It shows that unlike the (Euclidean) scattering transform, the map which recovers a signal $\mathbf{x}$ from its BLIS coefficients is Lipschitz continuous. Therefore, the BLIS module can be stably inverted. 

\subsection{Statement and Proof of Theorem \ref{thm: not injective}}

We first introduce some notation. For $v_i,v_j\in V$, we let $d(v_i,v_j)$ denote the unweighted path distance between $v_i$ and $v_j$. That is,  $d(v_i,v_j)$ is the smallest $k$ such that $A^k(v_i,v_j)\neq 0$, when $v_i\neq v_j$, and $d(v_i,v_i)=0$. We then define the diameter of $G$
 by $$\text{diam}(G)=\max_{v_i,v_j\in V}d(v_i,v_j).$$ With this notation, we may now state our theorem in detail and provide a proof. 
\begin{theorem}\label{thm: not injective} Let $\mathcal{W}_J=\{\Psi_j\}_{j=0}^J\cup\{\Phi_J\}$  be the wavelets $\mathcal{W}^{(2)}_J$ constructed in Section \ref{sec: wavelets}.
Suppose at least one of the following two conditions hold.
\begin{enumerate}
\item $G$ is a bipartite graph.
\item $g(t)$ is as in \eqref{eqn: g canonical} and $\text{diam}(G)> 2s_{J+1}$.
\end{enumerate}
Then there exist signals $\mathbf{x_1},\mathbf{x_2}$ such that $\mathbf{x_1}\neq \pm \mathbf{x_2}$, but 
$$
M\Psi_j\mathbf{x_1}=M\Psi_j\mathbf{x_2}\quad \text{for all } 0\leq j \leq J,
$$
and therefore, $\mathbf{x_1}$ and $\mathbf{x_2}$ have identical $m$-th order scattering coefficients for all $m\geq 1$.
\end{theorem}

\begin{proof}

Let us first consider the case where $G$ is bipartite. As in Section \ref{sec: notation}, let $\mathbf{v_1},\ldots,\mathbf{v_n}$ denote the eigenvalues of $L_N$ with $L_N\mathbf{v_i}=\omega_i\mathbf{v_i}$, $0=\omega_1<\omega_2\leq \ldots\leq\omega_n\leq2$. It is known (see, e.g., Lemma 1.7 of \citet{chung_1997}) that since $G$ is bipartite, we have $\omega_n=2$.

Since the function $g(t)$
 defined in Section \ref{sec: diffusion matrices} satisfies $g(0)=1$, $g(2)=0$, this implies that $T=g(L_N)$ has eigenvalues of $1$ and $0$. Moreover, since $K$ is similar to $T$, this implies that $K$  also has eigenvalues $0$ and $1$. That is, there exist vectors $\mathbf{u_1},\mathbf{u_2}\neq 0$ such that $K\mathbf{u_1}=\mathbf{u_1},$ $K\mathbf{u_2}=0$.

Let $\mathbf{x_1}=\mathbf{u_1}+\mathbf{u_2}$ and $\mathbf{x_2}=\mathbf{u_1}-\mathbf{u_2}$. By definition, neither $\mathbf{u_1}$ or $\mathbf{u_2}$ are the zero vector and therefore it is clear that $\mathbf{x_1}\neq \pm\mathbf{x_2}$. Thus, the proof will be complete once we show that $M\Psi_{j}\mathbf{x}_1=M\Psi_{j}\mathbf{x}_2$ for all $0\leq j\leq J$.

We first note that for $i=1,2$ and $1\leq j\leq J,$ we have $s_i,s_{i+1}>0$ and thus we have 
\begin{align*}\Psi_j\mathbf{x_i}&=K^{s_j}\mathbf{x_i}-K^{s_{j+1}}\mathbf{x_i}\\
&=K^{s_j}(\mathbf{u_1}\pm\mathbf{u_2})-K^{s_{j+1}}(\mathbf{u_1}\pm\mathbf{u_2})\\
&=(K^{s_j}\mathbf{u_1}-K^{s_{j+1}}\mathbf{u_1})\pm(K^{s_{j+1}}\mathbf{u_2}-K^{s_{j}}\mathbf{u_2})\\
&=\mathbf{u_1}-\mathbf{u_1}\pm(0-0)\\
&=0,
\end{align*}
which implies that $M\Psi_j\mathbf{x_1}=M\Psi_j\mathbf{x_2}=0$.

In the case where $j=0$, we have $\Psi_0=I-K$. Therefore,
\begin{align*}\Psi_0\mathbf{x_i}&=\mathbf{x_i}-K\mathbf{x_i}\\
&=\mathbf{u_1}\pm\mathbf{u_2} - K(\mathbf{u_1}\pm \mathbf{u_2})\\
&=
\mathbf{u_1}\pm\mathbf{u_2}-(\mathbf{u_1}\pm0)\\
&=\pm\mathbf{u_2}.
\end{align*}
Therefore, we also have $M\Psi_0\mathbf{x_1}=M\Psi_0\mathbf{x_2}$, which completes the proof under the assumption the graph is bipartite.

In the case where $\text{diam}(G) >  2s_{J+1}$ and $g(t)$ is as in \eqref{eqn: g canonical}, the proof is based on adapting the  techniques from \citet{iwen2019lower} which analyzed the instability of phase retrieval from locally supported measurements on $\mathbb{C}^n$ to the irregular geometry of a graph. (See also \citet{cahill2016phase} and \citet{cheng2021stable}.) The assumption that $\text{diam}(G) >  2s_{J+1}$ implies that there exist disjoint, non-empty subsets $S_1,S_2\subseteq V$ such that 

\begin{equation}\label{eqn: D}
\min_{v_1\in S_2, \, v_2\in S_2}  d(v_1,v_2) \geq 2s_{J+1}+1.
\end{equation}

We now define $\mathbf{x_1}$ and $\mathbf{x_2}$ by \begin{equation}\label{eqn: x expansion}
\mathbf{x_1}\coloneqq\boldsymbol{\delta}_{S_1}\pm\boldsymbol{\delta}_{S_2}\quad\text{and}\quad \mathbf{x_2}\coloneqq\boldsymbol{\delta}_{S_1}\pm\boldsymbol{\delta}_{S_2},
\end{equation}
where $\boldsymbol{\delta}_{S_1}$ is the indicator signal defined by $\boldsymbol{\delta}_{S_1}(v)=1$ if $v\in S_1$, $\boldsymbol{\delta}_{S_1}(v)=0$ otherwise, and $\boldsymbol{\delta}_{S_2}$ is defined similarly. 

The following lemma will imply that there is no overlap in the support of $\Psi_j\boldsymbol{\delta}_{S_1}$ with $\Psi_j\boldsymbol{\delta}_{S_2}$.
\begin{lemma}\label{lem: supports} 
Let $\mathbf{x}$ be a graph signal whose non-zero entries are contained a set $S\subseteq V$. Then for all integers $t \geq 0$, the support of $K^t\mathbf{x}$ is contained in the set $S_t=\{v\in V: \exists u\in S, d(u,v)\leq t\}.$
\end{lemma}
\begin{proof}
We first show that all  $K_{i,j}$ are zero except for  when either $i=j$ or $\{i,j\}\in E$. Indeed, this property is clearly satisfied by the unnormalized Laplacian $L_U=D-A$. Morever, the non-zero entries of a matrix are unchanged by multiplication (either on the left of the right) by a diagonal matrix. Therefore, this property is also satisfied by $L_N=D^{-1/2}L_UD^{-1/2}$. Additionally, this property is clearly satisfied by the identity matrix and also preserved under linear combinations. Therefore, it is also satisfied by $T=g(L_N)$. Lastly, we note that is also satisfied by $K$ since $K=W^{-1}TW$ and $W$
 is diagonal. 

 We now prove the lemma. The case where $t=0$ is trivial. 
 For $t=1$, we note that $\{K\mathbf{x}_i\}\neq 0$ implies that there exist $j$ such that $K_{i,j}\neq 0$ and $x_j\neq 0$. In light of the preceeding paragraph, this implies that $v_i$ is within distance one of some point in $S$ which proves the $t=1$ case. The result now follows inductively noting that $K^{t+1}\mathbf{x}=K(K^t\mathbf{x})$
\end{proof}
In light of Lemma \ref{lem: supports}, we see that for all $0\leq j\leq J$, the support of $\boldsymbol{\delta}_{S_1}$ is contained in $S_{1,s_{J+1}}\coloneqq\{v\in V: \exists u\in S_1, d(u,v)\leq s_{J+1}\}$ and the 
support of $\boldsymbol{\delta}_{S_2}$ is contained in $S_{2,s_{J+1}}\coloneqq\{v\in V: \exists u\in S_2, d(u,v)\leq s_{J+1}\}$.
Therefore, since $$\Psi_j(\mathbf{x_i})=\Psi_j\boldsymbol{\delta}_{S_1}\pm \Psi_j\boldsymbol{\delta}_{S_2},$$
\eqref{eqn: D} implies that 
$$\Psi_j(\mathbf{x_i})(v)=\begin{cases}
\Psi_j(\boldsymbol{\delta}_{S_1})(v) &\text{ if } v\in S_1\\
\pm \Psi_j(\boldsymbol{\delta}_{S_2})(v) &\text{ if } v\in S_2\\
0&\text{ otherwise}
\end{cases}.
$$
This implies that $M\Psi_j\mathbf{x_1}(v)=M\Psi_j\mathbf{x_2}(v)$ for all $j$ and all $v$ and therfore completes the proof.

\end{proof}

\begin{remark}
   In addition to $M\Psi_j\mathbf{x_1}=M\Psi_j\mathbf{x_2}$, we aslo have $M\Phi_J\mathbf{x_1}=M\Phi_J\mathbf{x_2}$. In the bipartite case, we have $\Phi_{J}\mathbf{x_1} =\Phi_{J}\mathbf{x_2}= \mathbf{u_1}$. Moreover, in the case where the graph has a large diameter, the fact that $M\Phi_J\mathbf{x_1}=M\Phi_J\mathbf{x_2}$ follows directly from Lemma \ref{lem: supports}.
\end{remark}

\begin{remark} In the large diameter case with $K=P$,
one may also modify the above construction  to ensure that the zeroth-order coefficients $\Phi_J\mathbf{x_i}$ are identical for $\mathbf{x_1}$ and $\mathbf{x_2}$, after they are fed into the aggregation module. To do this, we one chooses three sets $S_1,S_2,S_3$, which are all sufficiently far apart and satisfy $|S_1|=|S_2|=|S_3|$ and  similar to the proof above  set $\mathbf{x_1}=\boldsymbol{\delta}_{S_1}+\boldsymbol{\delta}_{S_2}-\boldsymbol{\delta}_{S_3}$, $\mathbf{x_2}=\boldsymbol{\delta}_{S_1}-\boldsymbol{\delta}_{S_2}+\boldsymbol{\delta}_{S_3}$. The fact that $P$ is a Markov matrix implies that if preserves the $\ell^1$ norm of each of the Dirac $\boldsymbol{\delta}$
 functions. Therefore,  one may verify that the aggregated zero-th order coefficients will be equal to $|S_1|$ for both signals. Additionally, we note that several papers on the graph scattering transform including \citet{gama:diffScatGraphs2018} and \citet{gao:graphScat2018} use global summation rather than low-pass filtering in the definition of the scattering coefficients, i.e., $\overline{S}_{J}[j_1,\ldots,j_m]\mathbf{x}=\|U[j_1,\ldots,j_m]\mathbf{x}\|_1$. This case, the zero's order coefficients of these two signals will also coincide, regardless of the choice of diffusion matrix. \end{remark}

\begin{remark}
Theorem \ref{thm: not injective} provides two examples of settings where the wavelet modulus fails to be injective. We note that the first assumption, that the graph is bipartite, is satisfied by graphs used in recommender systems where there are links between users and products. The latter assumption, that the graph has a large diameter will typically be satisfied by graphs constructed from geo-spatial data which do not experience the small world phenomenon. Our analysis shows that if graphs with large diameter, one needs to choose a large value of $J$ therefore increasing the computational cost of the wavelet transform. By contrast  Theorem \ref{prop: bi-Lipschitz} shows that in the BLIS module, $J$ may be chosen independent of the graph allowing for efficient implementation on graphs with large diameters.
\end{remark}

\section{Proof of Theorem \ref{prop: bi-Lipschitz}}\label{prf: bi-Lipschitz}

 Theorem \ref{prop: bi-Lipschitz} is proved by iteratively applying the frame bounds \eqref{eqn: Frame condition} as well as the following lemma which shows that the map $\sigma: \mathbb{R}^n\rightarrow\mathbb{R}^{2n}$ defined by $$\sigma(\mathbf{x})=(\sigma_1(\mathbf{x})^T,\sigma_2(\mathbf{x})^T)^T$$
is bi-Lipschitz.

\begin{lemma}\label{lem: sigmas bi-Lipschitz} 
For all $\mathbf{x}$ and $\mathbf{y}$ in $\mathbb{R}^n$, we have 
\begin{align*}
    \frac{1}{2} \norm{\mathbf{x}-\mathbf{y}}_{\mathbf{w}}^2 &\leq \norm{\sigma_{1}(\mathbf{x})-\sigma_{1}(\mathbf{y})}_{\mathbf{w}}^2+\norm{\sigma_{2}(\mathbf{x})-\sigma_{2}(\mathbf{y})}_{\mathbf{w}}^2\leq \norm{\mathbf{x}-\mathbf{y}}_{\mathbf{w}}^2.
\end{align*}
\end{lemma}
For a proof of Lemma \ref{lem: sigmas bi-Lipschitz}, please see Appendix \ref{prf: sigmas bi-Lipschitz}.

\begin{proof}[Proof of Theorem \ref{prop: bi-Lipschitz}] We argue by induction on $m$. To establish the base case $m=1$, we note that 
\begin{align*}
    & \norm{\mathbf{B}_{1}(\mathbf{x})-\mathbf{B}_{1}(\mathbf{y})}_{\mathbf{w},2}^2 
 = \sum_{k=1}^{2} \sum_{j=0}^{J+1} \norm{B[j,k](\mathbf{x})-B[j,k](\mathbf{y})}_{\mathbf{w}}^2  = \sum_{k=1}^{2}\sum_{j=0}^{J+1}  \norm{\sigma_{k} (F_{j}\mathbf{x})-\sigma_{k}(F_{j}\mathbf{y})}_{\mathbf{w}}^2.
\end{align*}
Lemma \ref{lem: sigmas bi-Lipschitz} and \eqref{eqn: Frame condition} imply
\begin{align*}
    \frac{c}{2}\|\mathbf{x}-\mathbf{y}\|_\mathbf{w} \leq 
    &  \frac{1}{2} \sum_{j=0}^{J+1} \norm{F_{j}\mathbf{x}-F_{j}\mathbf{y}}_{\mathbf{w}}^2 \\
& \leq \sum_{j=0}^{J+1} \sum_{k=1}^{2} \norm{\sigma_{k}(F_{j}\mathbf{x})-\sigma_{k}(F_{j}\mathbf{y})}_{\mathbf{w}}^2 \\
& \leq \sum_{j=0}^{J+1}\norm{F_{j}\mathbf{x}-F_{j}\mathbf{y}}_{\mathbf{w}}^2\\
&\leq    C \|\mathbf{x}-\mathbf{y}\|_\mathbf{w}, \end{align*}
which establishes the base case.

Now, assume the result for $m$, i.e.,
\begin{align*}
     \left( \frac{c}{2}\right) ^{m}\norm{\mathbf{x}-\mathbf{y}}_{\mathbf{w}}^2 & \leq \norm{\mathbf{B}_{m}(\mathbf{x})-\mathbf{B}_{m}(\mathbf{y})}_{\mathbf{w},2}^2\leq C^{m}\norm{\mathbf{x}-\mathbf{y}}_{\mathbf{w}}^2,
\end{align*}
and consider  $\norm{\mathbf{B}_{m+1}(\mathbf{x})-\mathbf{B}_{m+1}(\mathbf{y})}_{\mathbf{w},2}^2$. We note that by construction we have 
\begin{equation*}
B[j_1,k_1,\ldots,j_{m},k_m,j_{m+1},k_{m+1}](\mathbf{x})=B[j_{m+1},k_{m+1}]B[j_1,k_1,\ldots,j_{m},k_m](\mathbf{x}).
\end{equation*}

Therefore, for any fixed $j_1,k_1,\ldots,j_m,k_m$, we have 
\begin{align}
&\sum_{k_{m+1}=1}^2\sum_{j_{m+1}=0}^J\|B[j_1,k_1,\ldots,j_{m},k_m,j_{m+1},k_{m+1}](\mathbf{x})-B[j_1,k_1,\ldots,j_{m},k_m,j_{m+1},k_{m+1}](\mathbf{y})\|_\mathbf{w}^2\nonumber\\
=&\sum_{k_{m+1}=1}^2\sum_{j_{m+1}=0}^J\|B[j_{m+1},k_{m+1}]B[j_1,k_1,\ldots,j_{m},k_m](\mathbf{x})-B[j_{m+1},k_{m+1}]B[j_1,k_1,\ldots,j_{m},k_m](\mathbf{y})\|_\mathbf{w}^2\nonumber\\
\leq&C\|B[j_1,k_1,\ldots,j_{m},k_m](\mathbf{x})-B[j_1,k_1,\ldots,j_{m},k_m](\mathbf{y})\|_\mathbf{w}^2\label{eqn: UB induction step},
\end{align}  
where the final inequality follows by applying the result with $m=1$.
Similarly we have 
\begin{align}
&\sum_{k_{m+1}=1}^2\sum_{j_{m+1}=0}^J\|B[j_1,k_1,\ldots,j_{m},k_m,j_{m+1},k_{m+1}](\mathbf{x})-B[j_1,k_1,\ldots,j_{m},k_m,j_{m+1},k_{m+1}](\mathbf{y})\|_\mathbf{w}^2\nonumber\\
&\geq\frac{c}{2}\|B[j_1,k_1,\ldots,j_{m},k_m](\mathbf{x})-B[j_1,k_1,\ldots,j_{m},k_m](\mathbf{y})\|_\mathbf{w}^2.\label{eqn: LB induction step}
\end{align}  
Therefore, using the inductive hypothesis and \eqref{eqn: UB induction step}, we have 
\begin{align*}
    &\norm{\mathbf{B}_{m+1}(\mathbf{x})-\mathbf{B}_{m+1}(\mathbf{y})}_{\mathbf{w},2}^2 \\
    =& \sum_{k_{m+1}=1}^{2} \sum_{j_{m+1}=0}^{J+1}\sum_{k_{m}=1}^{2} \sum_{j_{m}=0}^{J+1} \cdots \sum_{k_{1}=1}^{2} \sum_{j_{1}=0}^{J+1}
     \norm{B[j_1,k_1,\ldots,j_{m},k_m,j_{m+1},k_{m+1}](\mathbf{x})-B[j_1,k_1,\ldots,j_{m},k_m,j_{m+1},k_{m+1}](\mathbf{y})}_{\mathbf{w}}^2\\
     \leq&C\sum_{k_{m}=1}^{2} \sum_{j_{m}=0}^{J+1} \cdots \sum_{k_{1}=1}^{2} \sum_{j_{1}=0}^{J+1}
     \norm{B[j_1,k_1,\ldots,j_{m},k_m,j_{m+1},k_{m+1}](\mathbf{x})-B[j_1,k_1,\ldots,j_{m},k_m,j_{m+1},k_{m+1}](\mathbf{y})}_{\mathbf{w}}^2\\
     \leq&C^{m+1}\|\mathbf{x}-\mathbf{y}\|_\mathbf{w}^2.
\end{align*}
which completes the proof for the upper bound. The lower bound follows by the same reasoning, but with \eqref{eqn: LB induction step} in place of \eqref{eqn: UB induction step}.

\end{proof}

\subsection{Inverting the BLIS Module}\label{app: inverse}

The lower bound in Theorem \ref{prop: bi-Lipschitz} implies the existence of a Lipschitz continuous inverse map (defined on the range of $\mathbf{B}_m$) which recovers $\mathbf{x}$ from $\mathbf{B}_m(\mathbf{x})$. This is noteworthy in part because because numerous works such as \citet{zou:graphScatGAN2019, BhaskarGCPK22,perlmutter2021hybrid, bruna2019multiscale} have attempted to invert variations of the scattering transform for the purposes of data synthesis (with varying degrees of theoretical justification).
We also note that in the the case where the wavelets are chosen to be $\mathcal{W}^{(2)}_J$ it is straightforward to invert each layer of the BLIS module since $\mathbf{x}=\sigma_1(\mathbf{x})-\sigma_2(\mathbf{x})$ and $\mathbf{x}=\sum_{j=0}^J\Psi^{(2)}_j\mathbf{x}+\Phi_J^{(2)}\mathbf{x}$. Indeed, in this setting, each layer of the BLIS module can essentially be viewed as a decomposition of the input signal somewhat analogous to wavelet packets (see e.g., \citet{coifman1992entropy}) .

\section{Proof of Theorem \ref{thm: equivariance}}\label{sec: proof equivariance}

\begin{proof}
Let $\Pi$ be a permutation matrix and let $A'$, $D'$, $L_N',$ $W'$, $K'$, etc, be the analogs of $A, D, L_N, W,$ and $K$ after the permutation. 

One may verify that $A'=\Pi A \Pi^T$, and $D'=\Pi A \Pi^T$ (where one $\Pi$ permutes the rows and the other permutes the columns). Since $\Pi^T\Pi=I=\Pi\Pi^T$, we see that $(D')^{1/2}=\Pi D^{1/2} \Pi^T$. Therefore,
$$
L_N'=I- (\Pi D^{-1/2}\Pi^T)(\Pi A \Pi^T)(\Pi D^{-1/2} \Pi^T)= \Pi L_N \Pi^T.
$$
Thus, we may compute
$$L_N'(\Pi\mathbf{v_i})=\Pi L_N \Pi^T(\Pi\mathbf{v_i})=\lambda_i\Pi \mathbf{v_i},$$
which implies that $\lambda_i'=\lambda_i,$ $\mathbf{v_i}'=\Pi\mathbf{v_i}$, and therefore that the eigendecomposition of $L_N'$ is given by 
$$
L_N' =(\Pi V) \Omega (\Pi V)^T.
$$
Thus, we have 
$$
T' = g(L_N') = (\Pi V) g(\Omega) (\Pi V)^T=(\Pi V) g(\Omega)  V^T \Pi^T= \Pi T \Pi^T.
$$
Additionally, for a suitably nice function $h$ (chosen to be either $p_j$ or $q_j$),
we have 
$$h(T')=\Pi V h(g(\Omega))  V^T \Pi^T=\Pi h(T)\Pi^T.$$
In particular, for either family of wavelets ($\mathcal{W}_J^{(1)}$ or $\mathcal{W}_J^{(2)})$, and we have $F_j'=\Pi F_j\Pi^T$, where, as in Section \ref{sec: BLIS}, $F_j$ is a generic member of the frame. Additionally, both $\sigma_1$ and $\sigma_2$ are element wise operators and thus commute with permutations. Therefore, we have 
\begin{align*}
B'[j_1,k_1,\ldots,j_m,k_m]&=\sigma_m(\Pi F_{j_m} \Pi^T(\sigma_{m-1}\ldots\sigma_1(\Pi F_{j_1} \Pi^T\cdot)\ldots)\\
&=\Pi B[j_1,k_1,\ldots,j_m,k_m]\Pi^T.
\end{align*}
This leads us to
$$
B'[j_1,k_1,\ldots,j_m,k_m]\Pi\mathbf{x}=\Pi B[j_1,k_1,\ldots,j_m,k_m]\mathbf{x}
$$
as desired.

\end{proof}

\section{Proof of Theorem \ref{prop: U nonexpansive}}\label{prf: U nonexpansive}

The proof of Theorem \ref{prop: U nonexpansive} is nearly identical to that of Theorem \ref{prop: bi-Lipschitz}, but relies on the following lemma in place of Lemma \ref{lem: sigmas bi-Lipschitz}. 

\begin{lemma}\label{lem: sigmas nonexpansive} %
For all $\mathbf{x} \in \mathbb{R}^n$, we have $$\norm{\sigma_{1}(\mathbf{x})}_{\mathbf{w}}^{2} + \norm{\sigma_{2}(\mathbf{x})}_{\mathbf{w}}^{2} = \norm{\mathbf{x}}_{\mathbf{w}}^{2}.$$
\end{lemma}
For a proof of Lemma \ref{lem: sigmas nonexpansive}, please see Appendix \ref{prf: sigmas nonexpansive}.


\begin{proof}[Proof of Theorem \ref{prop: U nonexpansive}] We proceed inductively on $m$. In the case $m=1$, we apply Lemma \ref{lem: sigmas nonexpansive} to see
\begin{align*}
    \norm{\mathbf{B}_{1}(\mathbf{x})}_{\mathbf{w},2}^2 
= \sum_{k=1}^{2} \sum_{j=0}^{J+1} \norm{B[j,k](\mathbf{x})}_{\mathbf{w}}^2 
= \sum_{j=0}^{J+1} \sum_{k=1}^{2} \norm{\sigma_{k}(F_{j}\mathbf{x})}_{\mathbf{w}}^2 = \sum_{j=0}^{J+1} \norm{F_{j}\mathbf{x}}_{\mathbf{w}}^2
\end{align*}

Therefore, by \eqref{eqn: Frame condition}
we have 
$$c\norm{\mathbf{x}}_{\mathbf{w}}^2 \leq \norm{\mathbf{B}_{1}(\mathbf{x})}_{\mathbf{w},2}^2\leq C\norm{\mathbf{x}}_{\mathbf{w}}^2,$$
which establishes the claim in the base case $m=1$.

Now, assume the result for $m$, i.e.,
$ c^{m}\norm{\mathbf{x}}_{\mathbf{w}}^2  \leq \norm{\mathbf{B}_{m}(\mathbf{x})}_{\mathbf{w},2}^2\leq C^{m}\norm{\mathbf{x}}_{\mathbf{w}}^2$,
and consider  $\norm{\mathbf{B}_{m+1}(\mathbf{x})}_{\mathbf{w},2}^2$. As in the proof of Theorem \ref{prop: bi-Lipschitz}, we have 
\begin{equation*}
B[j_1,k_1,\ldots,j_{m},k_m,j_{m+1},k_{m+1}](\mathbf{x})=B[j_{m+1},k_{m+1}]B[j_1,k_1,\ldots,j_{m},k_m](\mathbf{x}).
\end{equation*}

Therefore, for any fixed $j_1,k_1\ldots,j_m,k_m$, we have 
\begin{align}
\sum_{k_{m+1}=1}^2\sum_{j_{m+1}=0}^J\|B[j_1,k_1,\ldots,j_{m},k_m,j_{m+1},k_{m+1}](\mathbf{x})\|_\mathbf{w}^2
=&\sum_{k_{m+1}=1}^2\sum_{j_{m+1}=0}^J\|B[j_{m+1},k_{m+1}]B[j_1,k_1,\ldots,j_{m},k_m](\mathbf{x})\|_\mathbf{w}^2\nonumber\\
\leq&C\|B[j_1,k_1,\ldots,j_{m},k_m](\mathbf{x})\|_\mathbf{w}^2\nonumber,
\end{align}  
where the final inequality follows by applying the result with $m=1$.
Similarly, 
\begin{align}
\sum_{k_{m+1}=1}^2\sum_{j_{m+1}=0}^J\|B[j_1,k_1,\ldots,j_{m},k_m,j_{m+1},k_{m+1}](\mathbf{x})\|_\mathbf{w}^2\geq c\|B[j_1,k_1,\ldots,j_{m},k_m](\mathbf{x})\|_\mathbf{w}^2.\nonumber
\end{align}  
Therefore, using the inductive hypothesis, we have 

\begin{align*}
    \norm{\mathbf{B}_{m+1}(\mathbf{x})}_{\mathbf{w},2}^2
    =& \sum_{k_{m+1}=1}^{2} \sum_{j_{m+1}=0}^{J+1}\sum_{k_{m}=1}^{2} \sum_{j_{m}=0}^{J+1} \cdots \sum_{k_{1}=1}^{2} \sum_{j_{1}=0}^{J+1}
     \norm{B[j_1,k_1,\ldots,j_{m},k_m,j_{m+1},k_{m+1}](\mathbf{x})}_{\mathbf{w}}^2\\
     \leq&C\sum_{k_{m}=1}^{2} \sum_{j_{m}=0}^{J+1} \cdots \sum_{k_{1}=1}^{2} \sum_{j_{1}=0}^{J+1}
     \norm{B[j_1,k_1,\ldots,j_{m},k_m,j_{m+1},k_{m+1}](\mathbf{x})}_{\mathbf{w}}^2\\
     \leq&C^{m+1}\|\mathbf{x}\|_\mathbf{w}^2,
\end{align*}
and the lower bound follows similarly.

\end{proof}

\section{Proof of Lemma \ref{lem: sigmas bi-Lipschitz}}\label{prf: sigmas bi-Lipschitz}

\begin{proof} 
It suffices to show that for all $a,b\in\mathbb{R}$ we have
\begin{equation}\label{eqn: sigma bilipschitz on R}
\frac{1}{2}|a-b|^2\leq |\sigma_1(a)-\sigma_1(b)|^2 + |\sigma_2(a)-\sigma_2(b)|^2\leq |a-b|^2.
\end{equation}
For then we will have,
\begin{align}
&\frac{1}{2}\|\mathbf{x}-\mathbf{y}\|_\mathbf{w}^2\nonumber\\=&
\frac{1}{2}\sum_{i=1}^n|x_i-y_i|^2w_i\nonumber\\
\leq &
\sum_{i=1}^n|\sigma_1(x_i)-\sigma_1(y_i)|^2w_i + \sum_{i=1}^n|\sigma_2(x_i)-\sigma_2(y_i)|^2w_i\label{eqn: middle value}\\
\leq&\sum_{i=1}^n|x_i-y_i|^2w_i\nonumber\\
=& \|\mathbf{x}-\mathbf{y}\|_\mathbf{w}\nonumber,
\end{align}
which will complete the proof since the term from \eqref{eqn: middle value} is exactly $\|\sigma_1(\mathbf{x}-\mathbf{y})\|_\mathbf{w}^2+\|\sigma_1(\mathbf{x}-\mathbf{y})\|_\mathbf{w}^2$.

To prove \eqref{eqn: sigma bilipschitz on R}, we note that in the case where $a$ and $b$ have the same sign, then either $\sigma_1(a)=|a|$, $\sigma_1(b)=|b|$, and $\sigma_2(a)=\sigma_2(b)=0$ or $\sigma_2(a)=|a|$, $\sigma_2(b)=|b|$, and $\sigma_1(a)=\sigma_1(b)=0$. Either way, we have 
$$ 
|\sigma_1(a)-\sigma_1(b)|^2+|\sigma_2(a)-\sigma_2(b)|^2=|a-b|^2.
$$

In the case where $a$ and $b$ have different signs, assume without loss of generality that $a\geq 0 \geq b$.
Then, $|a-b|=|a|+|b|$ and so the result follows from noting 
\begin{equation*}
 |\sigma_1(a)-\sigma_1(b)|^2+|\sigma_2(a)-\sigma_2(b)|^2=|a|^2+|b|^2\geq \frac{1}{2}(|a|+|b|)^2\end{equation*}
 as well as the fact that $|a|^2+|b|^2\leq (|a|+|b|)^2.$
\end{proof}

\section{Proof of Lemma \ref{lem: sigmas nonexpansive}}\label{prf: sigmas nonexpansive}
\begin{proof}
Let $\mathbf{x}\in\mathbb{R}^n$. 
Let $\mathcal{I}\coloneqq\{i: x_i\neq 0\}$ and note that we may write $\mathcal{I}$ as the disjoint union $\mathcal{I}=\mathcal{I}_1\cup \mathcal{I}_2$ where $\mathcal{I}_1\coloneqq\{i: (\sigma_1(\mathbf{x}))_i\neq 0$\}, $\mathcal{I}_2\coloneqq\{i: (\sigma_1(\mathbf{x}))_i\neq 0$\}. Observe that for $j=1,2$ we have  $|(\sigma_j(\mathbf{x}))_i|^2=|x_i|^2$ whenever $i\in \mathcal{I}_j$.  Therefore,
\begin{align*}
\|\mathbf{x}\|^2_\mathbf{w}&
=\sum_{i\in \mathcal{I}_1} |x_i|^2w_i + \sum_{i\in \mathcal{I}_2} |x_i|^2w_i
=\sum_{i\in \mathcal{I}_1} |(\sigma_1(\mathbf{x}))_i|^2w_i + \sum_{i\in \mathcal{I}_2} |(\sigma_2(\mathbf{x}))_i|^2w_i
=\|\sigma_1(\mathbf{x})\|_\mathbf{w}^2+\|\sigma_2(\mathbf{x})\|_\mathbf{w}^2.\qedhere
\end{align*}
\end{proof}

\section{BLIS-Net Implementation and Computational Complexity}

Here we will discuss the implementation used in our experiments and also a modified implementation that can be used to increase the scalability of our network to large graphs with sparse connectivity.

In our experiments, we chose the diffusion operator $K=P=\frac{1}{2}(I + AD^{-1})$ with dyadic scales. When using $\mathcal{W}^{(2)}$, we computed the powers $P^{2^j}$ iteratively using the formula $P^{2^{j+1}}=P^{2^j}P^{2^j}$ and then computed the wavelets via subtraction. For the $\mathcal{W}^{(1)}_J$ wavelets, we computed an eigendecomposition of $T$ and then applied the functions $q_j$ along the diagonal. This simple implementation requires $\mathcal{O}(Jn^3)$  flops to construct the wavelet matrices for $\mathcal{W}^{(2)}_J$ and $\mathcal{O}(n^3+Jn)$  for $\mathcal{W}^{(1)}_J$. Then to perform the wavelet transform via matrix-vector multiplication we incur a cost of $\mathcal{O}(Jn^2)$ for each signal. Thus if there are $N$ signals in the data set, the total cost of the wavelet transform is $\mathcal{O}(Jn^3+Jn^2N)$. The memory cost of storing the wavelet matrices is $\mathcal{O}(Jn^2)$ for $\mathcal{W}^{(2)}_J$ and $\mathcal{O}(n^2+Jn)$ for $\mathcal{W}_J^{(1)}$. Since BLIS module consists of $m$ iterations of the wavelets followed by $\sigma_1$ and $\sigma_2$, it follows that 
the computational cost of an $m$-layer BLIS module is $\mathcal{O}(Jn^3+2^mJ^mn^2N)$. The memory requirements of storing the BLIS coefficients for $N$ signals is are $\mathcal{O}(2^mJ^mnN)$ (in addition to the memory costs of storing the wavelets). We also note the BLIS module is hand-crafted with no learnable parameters which means that these computations may be done offline as a preprocessing step. 

Based on this analysis, a simple implementation of BLIS is linear with respect to the number of signals and therefore is well-suited to scale in the setting where there are many different signals defined on a single moderate-size network (which is the primary focus of this work and is often the case in the context of signal-level tasks). 

It is also possible to modify our implementation to be scalable to large graphs. \citet{tong2022learnable} considered a modifies implentation of the wavlet transform which uses a diffusion module to compute $P\mathbf{x}, P^2\mathbf{x}, P^3\mathbf{x}\ldots,P^{2^J}\mathbf{x}$ via sparse matrix-vector multiplications and then compute the wavelets via vector-vector substraction. Notably,  \cite{wenkel2022overcoming} was able to use this method to achieve strong performance on large OGB benchmark data sets via a scattering-based network. If one implements BLIS in this manner, the computational cost is reduced to $O(2^{J + m} J^m (n + |E|))$  and the memory cost is reduced to  
$O(2^{m+J}J^mn)$ allowing for improved scalability.

\section{Models and Training}

\subsection{BLIS-Net architecture}\label{appendix:BLIS_net_arch}

A general and more complete description of the BLIS module and BLIS-Net architecture is given in Sections \ref{sec: BLIS} and \ref{sec: BLISnet}.
For all BLIS-Net experiments, we utilize dyadic scales and choose $J = 4$ meaning that our $\mathcal{W}_J^{(1)}$ and $\mathcal{W}_J^{(2)}$ wavelet filter banks both contain six filters that we apply to our signal. Furthermore, we fix $m = 3$ meaning that we only utilize third-order coefficients. Our moment aggregation module utilizes first-order moments across the nodes. Our embedding layer and classification layer implented as a single, unified MLP, where the choice of hidden layers was determined from the data using 5-fold cross validation on the testing set. The hidden layer sizes are chosen from the set: $[(50,), (100,), (50, 50), (150, 50)]$. Dimensionality reduction is achieved with the linear layer to the first hidden layer, and the classification is performed with a layer that maps from the final hidden layer dimension to the number of classes. The MLP utilizes ReLU activations in between layers, Adam optimizer, an L2 regularization term of 0.01, and all other default settings on scikit-learn's implementation of the MLP classifier. 

\subsection{Graph Scattering Transform}

For a complete description of the graph scattering transform, please refer to Section \ref{sec: scatttering}. For experiments involving the scattering transform, we utilize dyadic scales and choose $J=4$ meaning that our $\mathcal{W}_J^{(1)}$ and $\mathcal{W}_J^{(2)}$ wavelet filter banks both contain five filters that we apply to our signal (because unlike in BLIS we don't use the low-pass). 
Following the convention of \citet{gao:graphScat2018} we utilize zeroth-, first-, and second-order scattering coefficients unless otherwise specified. 
To make a fair comparison with the BLIS module, we use all combination of scales in the second-order coefficients (\citet{gao:graphScat2018} only used $j_2\geq j_1)$ and when performing aggregation we only utilize the first moments. (Notably, one could readily modify the aggregation module in BLIS-Net to include higher moments as well.) The back-end MLP shares an identical construction to the one described in \ref{appendix:BLIS_net_arch} for the BLIS-Net architecture. 

\subsection{Baseline Graph Neural Networks}

\textbf{Graph Convolutional Network (GCN)}: 
The baseline GCN \citep{kipf2016semi} consists of two GCNConv layers, both followed by a ReLU activation function. 
\begin{itemize}
    \item The first GCNConv layer transforms the input features to a hidden dimension of 16.
    \item The second GCNConv layer maintains this dimension, mapping from 16 to 16. 
    \item Following the convolutions, global mean pooling is applied to the node embeddings to obtain a graph-level representation.
    \item Finally, a linear layer is applied which outputs a dimension equal to the number of classes. 
\end{itemize}

\textbf{Graph Isomorphism Network (GIN)}:
The baseline GIN model \citep{xu2018how} is structured with two GINConv layers, each of which is associated with its own MLP.

\begin{itemize}
    \item The first GINConv layer utilizes an MLP that consists of two linear layers:
    \begin{enumerate}
        \item The initial layer transforms the input features to a hidden dimension of 16.
        \item The subsequent layer retains this dimensionality, taking in the 16-dimensional space and outputting another 16-dimensional space. Between these two layers, a ReLU activation function is applied.
    \end{enumerate}
    \item The second GINConv layer has a similar MLP structure, mapping the 16-dimensional space from the first layer to another 16-dimensional space, again with a ReLU activation function in between.
    \item After both GINConv layers process the node features, a global mean pooling aggregates these features to produce a graph-level representation.
    \item This graph-level representation is then processed by a linear layer, transforming from 16 dimensions to a dimensionality equal to the number of classes.
\end{itemize}

\textbf{Graph Attention Network (GAT)}:
The GAT \citep{velivckovic2017graph} baseline is structured with two GATConv layers, each employing an attention mechanism.

\begin{itemize}
    \item The first GATConv layer uses a single attention head, transforming the input features to a hidden dimension of 16.
    \item The second GATConv layer, operating in the same 16-dimensional space, continues this transformation, retaining the dimensionality of 16.
    \item An Exponential Linear Unit (ELU) activation function follows each of the convolutional layers.
    \item After the GATConv layers have processed the node features, a global mean pooling aggregates these features to produce a graph-level representation.
    \item This graph-level representation then undergoes a linear transformation, mapping from the 16-dimensional space to a dimensionality equal to the number of classes.
\end{itemize}

\textbf{General, Powerful, Scalable (GPS) Graph Transformer }
The GPS model \citep{rampavsek2022recipe} is designed to process graph-structured data with the incorporation of random walk-based positional encodings and the GPSConv layer.

Prior to feeding data into the model, random walk positional encodings of length 20 are added to the graph nodes.
\begin{itemize}
    \item \textbf{Embeddings}:
    The input features undergo a linear transformation to produce node embeddings. Additionally, positional encodings are normalized and transformed into a space of dimension 8.
    
    \item \textbf{GPSConv Layers}:
    The architecture employs two GPSConv layers, each featuring:
    \begin{itemize}
        \item A local message passing via GINEConv \citep{hu2019strategies}. This mechanism utilizes a two-layer MLP with hidden dimension 16 with ReLU activations.
        
        \item Multi-head attention with 4 heads and an attention dropout rate of 0.5.
    \end{itemize}
    
    \item \textbf{Classification}:
    Post the GPSConv processing, graph-level embeddings are obtained through a global addition pooling. These embeddings are directed into an MLP with 2 hidden layers to yield the final classification output.
\end{itemize}

\subsection{Training details}

All data sets are subjected to a 70/30 train-test split, and performance metrics are computed by averaging results over a 5-fold cross-validation. For the baseline models, namely the Graph Convolutional Network (GCN), Graph Attention Network (GAT), Graph Isomorphism Network (GIN), and the general, powerful, scalable (GPS) graph Transformer, the Adam optimizer is employed with a learning rate of 0.01. These models are trained for 100 epochs to ensure convergence.

BLIS-Net is trained using the default sci-kit learn training protocol for the MLPClassifier, with full training details available in the documentation.

\section{Additional Descriptions of the data sets}
Here we provide further descriptions of the data sets considered in our experiments and also provide summary statistics in Table \ref{tab:dataset_stats_appendix}.
\begin{table}[]
\centering 
\begin{adjustbox}{width=\linewidth}
\begin{tabular}{@{}l|llllll@{}}
\toprule
Data set statistics        & $|V|$ & $|E|$ & \begin{tabular}[c]{@{}l@{}}Number of \\ signals\end{tabular} & \begin{tabular}[c]{@{}l@{}}Number of \\ classes\end{tabular} & \begin{tabular}[c]{@{}l@{}}Signal \\ dimension\end{tabular} & \begin{tabular}[c]{@{}l@{}}Number of \\ sub-data sets\end{tabular} \\ \midrule
Partly Cloudy             & 39    & 113   & 168                                                          & 3                                                            & 1                                                           & 155                                                               \\ \midrule
Synthetic same $\mu$      & 100   & 358*  & 400                                                          & 2                                                            & 1                                                           & 5                                                                 \\
Synthetic different $\mu$ & 100   & 358*  & 400                                                          & 2                                                            & 1                                                           & 5                                                                 \\ \midrule
PEMSD3                    & 358   & 546   & 26208                                                        & 24,7,4                                                       & 1                                                           & 1                                                                 \\
PEMSD4                    & 307   & 340   & 16992                                                        & 24,7,4                                                       & 3                                                           & 1                                                                 \\
PEMSD7                    & 883   & 866   & 28224                                                        & 24,7,4                                                       & 1                                                           & 1                                                                 \\
PEMSD8                    & 170   & 274   & 17856                                                        & 24,7,4                                                       & 3                                                           & 1                                                                 \\ \bottomrule
\end{tabular}
\end{adjustbox}
\caption{Summary of the Data sets mentioned in the paper. For the traffic data set, the list under number of classes is specified in the context of a particular task. 24 corresponds to the HOUR task, 7 corresponds to the DAY task, and 4 corresponds to the WEEK task. In the case of the Partly Cloudy data set, the number of sub-data sets corresponds to the number of participants for the experiment, and each participant shares the same underlying graph. For the synthetic data set, the number of sub-data sets reflects that 5 replicates were conducted for each task, meaning that 400 unique signals were generated on each of 5 random graphs. This is done to characterize the variation depending on the random generation on the graph. Due to this, the asterisk next to 358 for the number of edges is reflective of the mode of the number of edges for the 5 replicates.}
\label{tab:dataset_stats_appendix}
\end{table}

\subsection{Further details on the traffic data sets }

To construct the signals from the PeMS data, we use the pre-processing procedure introduced in \citet{guo2019attention}. The graph structure is created by selecting sensors at least 3.5 miles apart and connecting adjacent sensors. Missing values in the graph signals are imputed using linear interpolation. PeMSD3, PeMSD4, PeMSD7, and PeMSD8 respectively consist of traffic data from California's 3rd, 4th, 7th, and 8th congressional districts and provide two months of consecutive traffic data collected between 2016 and 2018 depending on the data set. The PeMSD3 and PeMSD7 data sets we used contained a measurement of traffic flow at each sensor location. The PeMSD4 and PeMSD8 data sets used contain three types of measurements at each sensor location: total flow, average speed, and average occupancy. For these data sets, we pass each measurement into the BLIS module independently and then concatenate. 

\subsection{Further details on the Partly Cloudy data sets}

The "Partly Cloudy" data set, sourced from \citet{richardson2018development}, comprises MRI data captured from participants aged 3-12 years and adults as they watched the Disney Pixar animated film "Partly Cloudy". The data set's full title is "MRI data of 3-12 year old children and adults during viewing of a short animated film". The study involved 122 children and 33 adults. While undergoing the MRI scan, participants simply watched the film without any specific task.

The film is notable for portraying the characters' bodily sensations, such as pain, and their mental states. Movie frame annotations—categorizing them as positive, neutral, or negative in emotion—are derived from the labels in the repository of the paper \citet{rieck2020uncovering}.

For data preprocessing and Region of Interest (ROI) extraction, we utilized nilearn. We constructed a spatial connectivity graph from the ROI centroids, linking each centroid to its five closest neighbors. (We then symmetrize the graph by then setting $A_{i,j}=1$ if either $A_{i,j}=1$ or $A_{j,i}=1$). We then applied temporal smoothing to the time series data for each node, using a Gaussian kernel convolution with a $\sigma$ value of 1.75. Since fMRI data is extremely noisy, this temporal smoothing was critical for optimal model performance, as is explored in Table \ref{tab:partly_cloudy_t_smoothing_effect}.

\subsection{Further details on the Synthetic data set}

We consider the functions and graph generation methods described in \ref{sec: experiments synthetic}. We generate the nodes of the graph by sampling 100 points randomly from $[0,1]^2$ and connect each node to it's 5-nearest neighbors. As with the Partly Cloudy data, the graph is symmetrized if necessary. In total, 400 signals are generated per graph, with 200 signals corresponding to $f_j^{(1)}$ and 200 signals corresponding to $f_j^{(2)}$ to result in balanced classes. 
We generate 5 versions of this synthetic data set to control for randomness in the sampling of the vertices and the generation of signals.

\section{Ablation Study and Additional Experiments}\label{app: ablation}

BLIS-Net relies on pairing the BLIS module with an aggregation module, a dimension-reduction module (parametrized by an MLP), and a classification module (also parameterized by an MLP). However, one could also utilize the BLIS module in many other ways. In Tables \ref{tab:pems03_appendix}, \ref{tab:pems04_appendix}, \ref{tab:pems07_appendix},  \ref{tab:pems08_appendix}, \ref{tab: synthetic data long}, and  \ref{tab: fmriablation}, we show that the BLIS module can also be paired with shallow classifiers such as logistic regression (LR), random forest (RF), support vector classifier, and extreme gradient boosting (XGB) and also examine the performance of scattering with these same classifiers. (We also consider Scattering + MLP for direct comparison to BLIS-Net.) 
Notably, we perform these experiments on the PEMS04 and PEMS08 data sets in addition to those in the main body.

We see that BLIS-based methods generally perform well and consistently outperform the analogous scattering methods. For example, on the Partly Cloudy fMRI data set, BLIS-Net with the $\mathcal{W}^2_J$ wavelets and a simple logistic regression classifier is able to achieves $65.9\%$ accuracy whereas the corresponding scattering implementation achieves only $53.1\%$. On the traffic data sets, we note that BLIS + XGB has the overall best performance, usually slightly better than BLIS-Net. On the synthetic data and the fMRI data, BLIS-Net is the top performer, followed by BLIS + logistic regression.


\begin{table}[htbp]
\centering
\begin{adjustbox}{width=0.7\linewidth}
\begin{tabular}{l|ccc}
\toprule
PEMS03        & HOUR  & DAY   & WEEK  \\ \hline
GCN           & 27.8  & 14.1  & 30.8  \\
GAT           & 27.4  & 14.1  & 30.8  \\
GIN           & 14.0  & 14.3  & 30.8  \\ 
GPS           & 57.4  & 49.6  & 31.9  \\
\midrule
BLIS-Net (W1)   & $63.1$  & $53.1$ & $54.8$ \\
BLIS-Net (W2)   & $68.3$  & \best{$\mathbf{56.3}$} & \best{$\mathbf{61.7}$} \\
\midrule
BLIS + LR (W1)    & 49.0  & 37.0  & 43.4  \\
BLIS + LR (W2)    & 53.0  & 42.2  & 46.7  \\
BLIS + RF (W1)    & 63.4  & 52.3  & 52.9  \\
BLIS + RF (W2)    & 63.5  & 53.4  & 55.7  \\
BLIS + SVC (W1)   & 49.1  & 35.1  & 37.9  \\
BLIS + SVC (W2)   & 49.5  & 35.9  & 41.6  \\
BLIS + XGB (W1)   & \second{68.8}  & 54.0  & 52.6  \\
BLIS + XGB (W2)   & \best{69.2}  & \best{56.3}  & \second{56.2}  \\
Scattering + MLP (W1) & 58.2 & 45.6 & 46.4 \\ 
Scattering + MLP (W2) & 60.4 & 49.5 & 51.4 \\
Scattering + LR (W1)  & 42.3  & 33.1  & 37.8  \\
Scattering + LR (W2)  & 46.0  & 33.2  & 39.1  \\
Scattering + RF (W1)  & 56.0  & 44.8  & 43.9  \\
Scattering + RF (W2)  & 57.9  & 46.6  & 48.4  \\
Scattering + SVC (W1) & 43.5  & 28.3  & 32.6  \\
Scattering + SVC (W2) & 46.5  & 30.5  & 36.1  \\
Scattering + XGB (W1) & 56.7  & 42.8  & 41.6  \\
Scattering + XGB (W2) & 59.5  & 44.6  & 45.7  \\
\bottomrule
\end{tabular}
\end{adjustbox}
\caption{Accuracy on the PEMS03 traffic data set. \best{Best} and \second{second} best results are colored.}\label{tab:pems03_appendix}
\end{table}

\begin{table}[h]
\centering
\begin{adjustbox}{width=0.7 \linewidth}
\begin{tabular}{l|ccc}
\toprule
PEMS04        & HOUR & DAY & WEEK \\ \hline
GCN           & 38.1 & 19.5 & 28.6 \\
GAT           & 38.8 & 19.6 & 28.6 \\
GIN           & 39.8 & 17.7 & 28.6 \\
GPS           & 66.5 & 67.0 & 31.7 \\
\midrule
BLIS-Net (W1)   & 82.9 & 87.8 & 91.1 \\
BLIS-Net (W2)   & 84.2 & 91.9 & 92.3 \\
\midrule
BLIS + LR (W1)    & 74.7 & 71.4 & 69.5 \\
BLIS +  LR (W2)    & 71.5 & 69.4 & 68.2 \\
BLIS +  RF (W1)    & 82.4 & 89.8 & 90.9 \\
BLIS +  RF (W2)    & 80.7 & 88.5 & 89.5 \\
BLIS +  SVC (W1)   & 71.9 & 75.5 & 77.0 \\
BLIS +  SVC (W2)   & 69.5 & 73.5 & 75.6 \\
BLIS +  XGB (W1)   & \best{86.4} & \best{93.9}& \best{93.6} \\
BLIS +  XGB (W2)   & \second{86.1} & \second{92.8}& \second{92.9} \\
\midrule
Scattering  + MLP (W1) & 78.5 & 83.2 & 83.8 \\
Scattering  + MLP (W2) & 81.2 & 85.9 & 86.4 \\
Scattering  + LR (W1)  & 58.8 & 47.7 & 44.7 \\
Scattering  + LR (W2)  & 63.3 & 49.8 & 47.6 \\
Scattering  + RF (W1)  & 76.8 & 79.0 & 79.3 \\
Scattering  + RF (W2)  & 78.4 & 82.9 & 82.6 \\
Scattering  + SVC (W1) & 60.0 & 55.9 & 55.6 \\
Scattering  + SVC (W2) & 64.2 & 62.9 & 61.0 \\
Scattering  + XGB (W1) & 81.3 & 79.1 & 75.8 \\
Scattering + XGB (W2) & 82.6 & 82.9 & 79.8 \\
\bottomrule
\end{tabular}
\end{adjustbox}
\caption{Accuracy on the PEMS04 traffic data set.}
\label{tab:pems04_appendix}
\end{table}

\begin{table}[h]
\centering
\begin{adjustbox}{width=0.7\linewidth}
\begin{tabular}{l|ccc}
\toprule
PEMS07        & HOUR            & DAY            & WEEK           \\ \hline
GCN           & $27.4$  & $14.6$ & $28.5$ \\
GAT           & $26.8$  & $14.6$ & $28.6$ \\
GIN           & $14.3$ & $15.8$ & $28.4$ \\
GPS           &  $39.9$ &  $27.7$ & $30.4$ \\
\midrule
BLIS-Net (W1)   &  ${63.5}$  & $\second{72.9}$ & $\second{76.8}$ \\
BLIS-Net (W2)   & ${63.4}$  & ${71.0}$ & $\best{77.3}$ \\
\midrule
BLIS +  LR (W1)    & $47.3$  & $46.2$ & $54.5$ \\
BLIS +  LR (W2)    & $43.4$  & $41.0$ & $51.3$ \\
BLIS +  RF (W1)    & $60.7$  & $67.7$ & $71.9$ \\
BLIS +  RF (W2)    & $57.4$  & $62.7$ & $67.6$ \\
BLIS +  SVC (W1)   & $51.1$  & $53.7$ & $56.7$ \\
BLIS +  SVC (W2)   & $43.2$  & $41.8$ & $46.9$ \\
BLIS +  XGB (W1)   & $\best{68.6}$  & $\best{74.7}$ & $75.3$ \\
BLIS +  XGB (W2)   & $\second{64.8}$  & $66.5$ & $69.1$ \\
Scattering +  MLP (W1) & $54.0$ & $53.3$ & $56.9$ \\ 
Scattering  + MLP (W2) & $54.3$ & $55.2$ & $61.6$ \\
Scattering  + LR (W1)  & $36.7$  & $33.3$ & $39.4$ \\
Scattering  + LR (W2)  & $36.7$  & $29.9$ & $40.4$ \\
Scattering  + RF (W1)  & $53.5$  & $51.9$ & $52.6$ \\
Scattering  + RF (W2)  & $52.7$  & $53.4$ & $56.5$ \\
Scattering  + SVC (W1) & $39.7$  & $35.5$ & $38.7$ \\
Scattering  + SVC (W2) & $40.6$  & $34.7$ & $42.0$ \\
Scattering  + XGB (W1) & $53.2$  & $50.2$ & $49.1$ \\
Scattering  + XGB (W2) & $54.1$  & $50.9$ & $52.6$ \\
\bottomrule
\end{tabular}
\end{adjustbox}
\caption{Accuracy on the PEMS07 traffic data set.}
\label{tab:pems07_appendix}
\end{table}

\begin{table}[h]
\centering
\begin{adjustbox}{width=0.7\linewidth}
\begin{tabular}{l|ccc}
\toprule
PEMS08        & HOUR & DAY & WEEK \\ \hline
GCN           & 33.3 & 20.4 & 32.3 \\
GAT           & 31.9 & 21.5 & 32.4 \\
GIN           & 24.5 & 14.5 & 32.1 \\
GPS           & 67.7 & 67.9 & 62.3 \\
\midrule
BLIS-Net (W1)   & 83.9 & 92.9 & 93.4 \\
BLIS-Net (W2)   & 85.9 & 94.9 & \second{95.6} \\
\midrule
BLIS + LR (W1)    & 69.5 & 76.5 & 78.1 \\
BLIS +  LR (W2)    & 71.9 & 80.2 & 81.8 \\
BLIS +  RF (W1)    & 83.7 & 92.7 & 91.6 \\
BLIS +  RF (W2)    & 83.9 & 93.5 & 93.6 \\
BLIS +  SVC (W1)   & 72.4 & 85.2 & 84.9 \\
BLIS +  SVC (W2)   & 73.8 & 87.4 & 89.5 \\
BLIS +  XGB (W1)   & \second{87.2} &\second{95.1} & 94.7 \\
BLIS +  XGB (W2)   & \best{87.7} & \best{96.0} & \best{96.1} \\
\midrule
Scattering + MLP (W1) & 81.0 & 89.9 & 89.3 \\
Scattering + MLP (W2) & 82.2 & 92.0 & 90.7 \\
Scattering  + LR (W1)  & 56.6 & 56.1 & 54.1 \\
Scattering  + LR (W2)  & 58.6 & 60.1 & 57.8 \\
Scattering  + RF (W1)  & 79.2 & 86.0 & 84.2 \\
Scattering  + RF (W2)  & 80.6 & 88.1 & 87.9 \\
Scattering  + SVC (W1) & 63.4 & 71.5 & 66.2 \\
Scattering  + SVC (W2) & 66.1 & 76.6 & 72.0 \\
Scattering  + XGB (W1) & 81.8 & 87.2 & 81.9 \\
Scattering  + XGB (W2) & 83.6 & 89.6 & 86.1 \\
\bottomrule
\end{tabular}
\end{adjustbox}
\caption{Accuracy on the PEMS08 traffic data set.}
\label{tab:pems08_appendix}
\end{table}

\begin{table}[]
\centering
\begin{adjustbox}{width=0.7\linewidth}
\begin{tabular}{@{}l|cc@{}}
\toprule
Synthetic     & Different $\mu$ & Same $\mu$     \\ \midrule
GCN           & $99.0 \pm 0.4$  & $91.7 \pm 2.0$ \\
GAT           & $99.2 \pm 0.5$  & $91.6 \pm 2.0$ \\
GIN           & $99.5 \pm 0.2$  & $91.3 \pm 1.4$ \\
GPS           & $95.4 \pm 5.9$  & $97.7 \pm 0.9$ \\
 \midrule
 BLIS-Net (W1) & \best{$\mathbf{100.0 \pm 0.0}$} & $97.7 \pm 0.5$ \\
 BLIS-Net (W2) & $99.5 \pm 0.3$  & \second{$\mathbf{98.6 \pm 0.4}$} \\
 \midrule
 BLIS + LR (W1) & \best{$\mathbf{100.0 \pm 0.0}$} & $98.5 \pm 1.0$ \\
 BLIS  + LR (W2) & \best{$\mathbf{100.0 \pm 0.0}$} & \best{$\mathbf{98.8 \pm 0.4}$} \\
 BLIS  + RF (W1) & $99.2 \pm 0.4$  & $97.7 \pm 0.6$ \\
 BLIS  + RF (W2) & $99.4 \pm 0.1$  & $97.1 \pm 0.7$ \\
 BLIS  + SVC (W1) & $99.4 \pm 0.1$ & $95.7 \pm 1.5$ \\
 BLIS +  SVC (W2) & \best{$\mathbf{100.0 \pm 0.0}$} & $95.5 \pm 1.9$ \\
 BLIS  + XGB (W1) & $99.5 \pm 0.3$ & $98.4 \pm 0.1$ \\
 BLIS  + XGB (W2) & $99.3 \pm 0.0$ & $97.7 \pm 0.7$ \\
 \midrule
 Scattering  + MLP (W1) & $97.7 \pm 1.0$ & $96.5 \pm 1.2$ \\ 
Scattering  + MLP (W2) & $88.3 \pm 4.3$ & $96.8 \pm 1.0$ \\ 
Scattering  + LR (W1) & $97.7 \pm 0.7$ & $96.1 \pm 1.0$ \\
Scattering  + LR (W2) & $86.9 \pm 4.9$ & $95.3 \pm 1.5$ \\
Scattering  + RF (W1) & $95.9 \pm 1.6$ & $94.5 \pm 1.6$ \\
Scattering  + RF (W2) & $73.4 \pm 8.0$ & $94.1 \pm 1.6$ \\
Scattering  + SVC (W1) & $98.1 \pm 0.9$ & $95.0 \pm 1.5$ \\
Scattering  + SVC (W2) & $87.5 \pm 4.9$ & $93.5 \pm 1.8$ \\
Scattering  + XGB (W1) & $95.2 \pm 1.9$ & $93.1 \pm 2.5$ \\
Scattering  + XGB (W2) & $81.9 \pm 7.2$ & $94.7 \pm 1.5$ \\
\bottomrule
\end{tabular}
\end{adjustbox}
\caption{Accuracy on the synthetic data sets. }\label{tab: synthetic data long}

\end{table}

\subsection{Denoising in the fMRI data set}

fMRI data is extremely noisy. Therefore, in our experiments on the fMRI data, we performed a Gaussian smoothing over the time variable. Importantly, we note that we applied the same smoothing procedure for all methods. Results with and without the smoothing are shown in Table \ref{tab:partly_cloudy_t_smoothing_effect}. We see that without the smoothing all methods perform poorly, with GPS being the top performing method at $42.0\%$ followed closely by BLIS-Net (W2) at $41.5\%$.  After the smoothing, the message passing networks (GCN, GAT, and GIN) continue to perform poorly (at most $42.1\%$). GPS improves from $42.0$ to $56.4\%$, scattering improves from $40.3 / 40.7 \%$ to $60.6 / 62.3\%$ and BLIS-Net improves from $41.1 / 41.5\%$ to $67.1 / 68.3\%$.

\begin{table}[]
\centering
\begin{adjustbox}{width=0.7\linewidth}
\begin{tabular}{@{}l|c@{}}
\toprule
Partly Cloudy & Emotion classification \\ \midrule
GCN           & $39.3 \pm 5.9$         \\
GAT           & $39.3 \pm 6.0$        \\
GIN           & $42.1 \pm 6.0$        \\ 
GPS           & $56.4 \pm 4.3$ \\
\midrule
BLIS-Net (W1)       & \second{$\mathbf{67.1 \pm 4.3}$}        \\
BLIS-Net (W2)       & \best{$\mathbf{68.3 \pm 3.6}$}         \\
\midrule
BLIS +  LR (W1) & $62.4 \pm 5.4$ \\
BLIS +  LR (W2) & $65.9 \pm 5.2$ \\
BLIS +  RF (W1) & $61.5 \pm 5.2$ \\
BLIS +  RF (W2) & $63.0 \pm 4.5$ \\
BLIS +  SVC (W1) & $56.2 \pm 5.3$ \\
BLIS +  SVC (W2) & $59.0 \pm 5.0$ \\
BLIS +  XGB (W1) & $61.1 \pm 5.7$ \\
BLIS +  XGB (W2) & $62.8 \pm 5.1$ \\
\midrule
Scattering +  MLP (W1) & $60.6 \pm 4.9$ \\ 
Scattering + MLP (W2) & $62.3 \pm 5.1$ \\ 
Scattering + LR (W1) & $51.2 \pm 5.8$ \\
Scattering +  LR (W2) & $53.1 \pm 5.9$ \\
Scattering  + RF (W1) & $56.1 \pm 6.0$ \\
Scattering  + RF (W2) & $58.8 \pm 5.5$ \\
Scattering  + SVC (W1) & $51.5 \pm 5.9$ \\
Scattering  + SVC (W2) & $54.2 \pm 6.1$ \\
Scattering  + XGB (W1) & $56.2 \pm 6.0$ \\
Scattering +  XGB (W2) & $58.3 \pm 5.8$ \\
\bottomrule
\end{tabular}
\end{adjustbox}
\caption{ Accuracy on Partly Cloudy fMRI data.  }
\label{tab: fmriablation}
\end{table}

\begin{table}[]
\centering
\begin{adjustbox}{width=\linewidth}
\begin{tabular}{@{}l|c|c@{}}
\toprule
Partly Cloudy & Emotion classification (No smoothing) & Emotion classification (with smoothing) \\ \midrule
GCN           & $37.5 \pm 4.9$                       & $39.3 \pm 5.9$         \\
GAT           & $37.3 \pm 4.8$                       & $39.3 \pm 6.0$        \\
GIN           & $37.1 \pm 4.5$                       & $42.1 \pm 6.0$        \\ 
GPS           & \best{$\mathbf{42.0 \pm 4.3}$}      & $56.4 \pm 4.3$ \\
\midrule
Scattering (W1) & $40.3 \pm 5.3$                     & $60.6 \pm 4.9$ \\ 
Scattering (W2) & $40.7 \pm 5.8$                     & $62.3 \pm 5.1$ \\ 
\midrule
BLIS-Net (W1)  & \third{${41.1 \pm 5.0}$}            & \second{$\mathbf{67.1 \pm 4.3}$}        \\
BLIS-Net (W2)  & \second{$\mathbf{41.5 \pm 5.5}$}    & \best{$\mathbf{68.3 \pm 3.6}$}         \\
\bottomrule
\end{tabular}
\end{adjustbox}
\caption{ Effect of Gaussian temporal smoothing on the accuracy on Partly Cloudy fMRI data. }
\label{tab:partly_cloudy_t_smoothing_effect}
\end{table}

\end{document}


%

%

\onecolumn
\aistatstitle{
Supplementary Materials for BLIS-Net: Classifying and Analyzing Signals on Graphs}